\documentclass[twoside,11pt]{article}

\usepackage[preprint]{jmlr2e}
\usepackage{lastpage}

\usepackage{amsmath,amssymb,mathtools}
\usepackage{bm}
\usepackage{graphicx}
\usepackage{subcaption}
\usepackage{booktabs}
\usepackage{multirow}
\usepackage{placeins}
\usepackage{algorithm}
\usepackage{algorithmic}
\usepackage{xcolor}
\usepackage{tikz}
\usetikzlibrary{arrows.meta, positioning, calc, shadows}
\usepackage{cleveref}

\hypersetup{hidelinks}

\DeclareMathOperator{\Cov}{Cov}
\DeclareMathOperator{\tr}{tr}
\DeclareMathOperator{\mmse}{mmse}

\newcommand{\R}{\mathbb{R}}
\newcommand{\E}{\mathbb{E}}
\newcommand{\PP}{\mathbb{P}}
\newcommand{\Nc}{\mathcal{N}}
\newcommand{\Cc}{\mathcal{C}}
\newcommand{\dd}{\mathrm{d}}

\newcommand{\Mc}{\mathcal{M}}
\newcommand{\KL}{\mathrm{KL}}
\newcommand{\Law}{\operatorname{Law}}

\newcommand{\mainlabel}[1]{{\sffamily\small\bfseries\color{black!80} #1}}
\newcommand{\sublabel}[1]{{\sffamily\footnotesize\color{gray!80} #1}}

\numberwithin{equation}{section}

\newtheorem{assumption}{Assumption}[section]

\jmlrheading{1}{2026}{1-\pageref{LastPage}}{4/26}{}{0000}
{Ahmad Aghapour, Erhan Bayraktar, and Asaf Cohen}

\ShortHeadings{Conditional Diffusion Under Linear Constraints}
{Aghapour, Bayraktar, and Cohen}

\firstpageno{1}

\begin{document}

\title{Conditional Diffusion Under Linear Constraints: Langevin Mixing and Information-Theoretic Guarantees}

\author{
\name Ahmad Aghapour \email aghapour@umich.edu \\
\addr Department of Mathematics \\
University of Michigan \\
Ann Arbor, MI 48109, USA
\AND
\name Erhan Bayraktar \email erhan@umich.edu \\
\addr Department of Mathematics \\
University of Michigan \\
Ann Arbor, MI 48109, USA
\AND
\name Asaf Cohen \email asafc@umich.edu \\
\addr Department of Mathematics \\
University of Michigan \\
Ann Arbor, MI 48109, USA
}

\editor{}
\maketitle

\begin{abstract}%
We study zero-shot conditional sampling with pretrained diffusion models for linear inverse problems, including inpainting and super-resolution. In these problems, the observation determines only part of the unknown signal. The remaining degrees of freedom must be sampled according to the correct conditional data distribution. Existing projection-based samplers enforce measurement consistency by correcting the observed component during reverse diffusion. However, measurement consistency alone does not determine how probability mass should be distributed along the feasible set, and this can lead to biased conditional samples.

We analyze this issue through a normal--tangent decomposition of the score function. For Gaussian noising, the observed-direction score is exactly determined by the measurement; only the tangent conditional score is unknown. We prove that the error from replacing this score by the unconditional tangent score is upper bounded by a dimension-free conditional mutual information between observed and unobserved components. This gives an information-theoretic decomposition into initialization and pathwise score-mismatch errors. Motivated by the theory, we propose a projected-Langevin initialization followed by guided reverse denoising, which outperforms a strong projection-based baseline in inpainting and super-resolution experiments.
\end{abstract}

\section{Introduction}

Diffusion models have become a standard tool for high-dimensional generative
modeling. Given samples from a data distribution, a diffusion model learns the
score of progressively noised versions of the data and then generates new
samples by simulating a reverse-time denoising process
\citep{songermon2019score, ho2020ddpm, song2021scorebased, song2020ddim}.
In many applications, however, generation is not unconditional. In image
restoration, for example, one observes a corrupted image and wants to sample
clean images that are both realistic under a pretrained image prior and
consistent with the observation.

This paper studies such conditional sampling problems for noiseless linear
observations. Let \(Z\in\R^d\) denote the clean signal and suppose that
\[
    y=AZ,\qquad A\in\R^{m\times d}.
\]
The goal is to sample from the conditional law
\[
    \Law(Z\mid AZ=y).
\]
When \(A\) has full row rank, the constraint \(AZ=y\) defines an affine set.
Writing
\[
P_{\perp}:=A^\top(AA^\top)^{-1}A,
\qquad
P_{\parallel}:=I-P_{\perp},
\]
the projection \(P_\perp\) extracts the component of the signal determined by
the measurements, while \(P_\parallel\) extracts the component in the null space
of \(A\). We refer to these as the normal and tangent components, respectively.
Thus the observation fixes the normal component, whereas the tangent component
contains the remaining degrees of freedom. In imaging problems, this formulation
covers inpainting, super-resolution, deblurring, and other linear inverse
problems.

A central difficulty is that measurement consistency and conditional sampling
are not the same task. Measurement consistency only requires producing a sample
\(\hat z\) satisfying \(A\hat z=y\). Conditional sampling requires more: among
all feasible signals satisfying the measurement, samples should be distributed
according to the true conditional law of the data. In the geometric language
above, the normal component enforces feasibility, while the tangent component
determines how probability mass is distributed along the feasible affine set.

Many zero-shot inverse-problem samplers based on pretrained diffusion models
enforce the measurement by repeatedly correcting or projecting the sample in the
observed directions. We call such methods projection-based because they use the
known linear operator \(A\) to replace, project, or analytically correct the
normal component during reverse diffusion, while leaving the unobserved
directions largely governed by the pretrained unconditional model. Methods such as denoising diffusion restoration models (DDRM) and the denoising
diffusion null-space model (DDNM) are representative examples for linear inverse
problems \citep{kawar2022ddrm, wang2022ddnm}. These methods can be very effective at
maintaining measurement consistency. However, correcting the normal component
does not by itself determine the correct distribution in the tangent directions.
As a result, a sample may satisfy \(A\hat z=y\) while still being biased along
the feasible manifold.

The goal of this paper is to understand and reduce this tangent-space bias. We
work in the zero-shot setting: the diffusion model is pretrained
unconditionally, is not fine-tuned for the observation, and conditioning is
imposed only at inference time. This setting is practically important because it
allows a single generative prior to be reused across many inverse problems. It
is also theoretically revealing, because the only available learned object is
the unconditional score. The question is therefore: when can an unconditional
score be used to approximate the conditional dynamics, and where does the error
enter?

Our starting point is a normal--tangent decomposition of the conditional score.
Under Gaussian noising, the normal component of the conditional score is
available in closed form from the observation. In the variance-exploding
normalization, if \(B=P_\perp Z=b\), then
\[
    P_\perp s_t^{*,b}(x)=\frac{1}{t}P_\perp(b-x).
\]
Thus the normal score is not the obstacle. The only unknown part is the tangent
conditional score \(P_\parallel s_t^{*,b}(x)\). Projection-based zero-shot
samplers can therefore be viewed as replacing this unknown tangent conditional
score by the pretrained unconditional tangent score \(P_\parallel s_t(x)\). This
view isolates the precise source of bias: the approximation is made along the
feasible directions, not in the measured directions.

Motivated by this decomposition, we propose a two-stage conditional sampler.
Rather than starting reverse diffusion from the highest-noise distribution, we
start from an intermediate noise level. At this level, the noisy normal
component can be sampled exactly under the constraint. We then run projected
underdamped Langevin dynamics on the corresponding affine slice, using the
projected unconditional score to mix only in the tangent directions. This
produces an initialization that is already consistent with the noisy constraint
and better adapted to the feasible slice. From this initialization, we perform
guided reverse denoising using the exact normal correction and the pretrained
unconditional score in the tangent directions.

The theoretical analysis follows the same decomposition. We separate the total
sampling error into two terms. The first is an initialization error at the
intermediate noise level, caused by approximating the true conditional marginal
on the affine slice. The second is a pathwise error accumulated during reverse
denoising, caused by replacing the true conditional tangent score with the
unconditional tangent score. Our main pathwise result shows that this second
error is controlled by a conditional mutual information between tangent and
normal components. Informally, zero-shot tangent guidance is accurate when, at
the chosen noise level, the remaining statistical dependence between the
unobserved tangent component and the observed normal component is small.

We further combine the pathwise bound with an initialization analysis. Under a
latent Gaussian-mixture model, we obtain a terminal Kullback--Leibler (KL) bound consisting of an
initialization term and a mutual-information pathwise term. Under an additional
separation condition on the latent normal codebook, both terms become
exponentially small in the separation-to-noise ratio. These results identify
regimes in which inference-time conditioning with a fixed unconditional score
can be accurate, and they also explain why tangent-space ambiguity is the
central obstruction.

We evaluate the resulting sampler on standard linear imaging inverse problems
using pretrained diffusion backbones and matched compute budgets. On inpainting
and \(8\times\) super-resolution across CelebA-HQ, LSUN Church, and ImageNet,
the proposed method improves Learned Perceptual Image Patch Similarity (LPIPS)
and Fréchet Inception Distance (FID) over a strong projection-based zero-shot
baseline. The gains are largest in settings with greater unresolved
tangent ambiguity, such as ImageNet and high-factor super-resolution, consistent
with the role of tangent mixing in the analysis.

\subsection{Related Work}

Our work builds on score-based generative modeling and diffusion models
\citep{songermon2019score, ho2020ddpm, song2021scorebased, song2020ddim}.
Conditional generation can be obtained by training conditional models, but many
inverse problems require reusing a fixed unconditional model. Classifier
guidance and classifier-free guidance modify the reverse process using
additional conditional information \citep{dhariwal2021guided,
ho2022classifierfree}. Image editing and restoration methods such as SDEdit,
RePaint, and ILVR impose conditioning through noising, denoising, and resampling
procedures \citep{meng2022sdedit, lugmayr2022repaint, choi2021ilvr}.

Diffusion priors have also been widely used for inverse problems. Predictor--
corrector samplers and likelihood-gradient corrections incorporate observations
during sampling \citep{song2021scorebased, song2022medinv}. For linear inverse
problems, DDRM and DDNM exploit the measurement operator to impose analytic
updates or null-space corrections during reverse diffusion
\citep{kawar2022ddrm, wang2022ddnm}. Diffusion posterior sampling (DPS) extends posterior sampling ideas to
more general noisy and nonlinear settings through likelihood-gradient guidance
\citep{chung2023dps}. These methods demonstrate the strength of pretrained
diffusion priors for restoration. Our focus is different: we analyze the
specific tangent-score approximation that remains after the normal measurement
component has been corrected.

Other approaches construct conditional samplers by changing the model or the
underlying path measure. Reward-based fine-tuning and reinforcement-learning
methods adapt a pretrained generator using task-specific feedback
\citep{fan2023dpok,zhao2025ctrl,uehara2024understanding}. Doob's
\(h\)-transform and diffusion-bridge methods provide principled path-space
formulations of conditioning
\citep{didi2023framework,guo2026hardconstraint,zhou2024ddbm}. These methods can
be exact or asymptotically exact under suitable assumptions, but they typically
require learning an additional object, solving a control problem, or fine-tuning
the model. By contrast, we keep the unconditional score fixed and study what can
be achieved by inference-time conditioning alone.

The Langevin initialization used here is related to constrained sampling.
Projected Langevin methods sample on constrained domains or manifolds
\citep{lamperski2021projected}. Underdamped Langevin dynamics can improve
mixing relative to overdamped dynamics in some settings
\citep{cheng2018underdamped}, and BAOAB discretizations are known for stable
and low-bias behavior in the position marginal
\citep{leimkuhler2013rational}. In affine inverse problems, these methods are
natural because, once the normal component is fixed, the remaining sampling
problem lives in the tangent space.

Recent theory has begun to analyze conditional and zero-shot diffusion samplers,
including asymptotically exact conditional samplers \citep{wu2023tds},
filtering-based posterior samplers for linear inverse problems
\citep{dou2024diffusion}, and score-mismatch analyses for zero-shot guidance
\citep{liang2025theory}. Our contribution is complementary: we isolate the
normal--tangent structure of affine conditioning and bound the error caused by
using the unconditional tangent score in place of the conditional tangent score.

\subsection{Contributions}

The main contributions of this paper are as follows.

First, we derive a normal--tangent decomposition of affine conditional
diffusion. For Gaussian noising, the normal component of the conditional score
is available exactly from the measurement and the noising process, while the
tangent component is the only part not supplied by a pretrained unconditional
score model. This decomposition motivates the surrogate dynamics in
Section~\ref{sec:methodology}.

Second, we propose a zero-shot conditional sampler that combines exact normal
correction, projected underdamped Langevin mixing on an affine slice, and guided
reverse denoising. The Langevin phase is designed to initialize the sampler at
an intermediate noise level with improved mixing in the unobserved tangent
directions before the final denoising stage.

Third, we prove a pathwise error bound for the guided reverse dynamics. In
Theorem~\ref{thm:avg-path-kl}, the KL divergence between the ideal conditional
path measure and the surrogate path measure is controlled by a conditional
mutual information between the tangent and normal components. This gives an
information-theoretic criterion for when replacing the conditional tangent score
by the unconditional tangent score is accurate.

Fourth, we combine the pathwise estimate with an initialization analysis to
obtain terminal KL guarantees. In
Theorem~\ref{thm:total_KL_shannon_small}, the terminal error separates into an
initialization term and the mutual-information pathwise term. The resulting
bound has no explicit dependence on the ambient dimension; its size is governed
by the sensitivity of tangent conditionals and by the residual statistical
dependence between observed and unobserved components. Under an additional
separation condition on the latent normal codebook,
Theorem~\ref{thm:total_error_sep_renyi_single} further gives an exponential
small-error regime, where both the initialization and pathwise contributions
become exponentially small  when component of gaussian mixture model is separated.

Finally, we evaluate the proposed sampler on linear imaging inverse problems.
The experiments show that the algorithm outperforms previous zero-shot
diffusion methods on inpainting and \(8\times\) super-resolution under matched
network-evaluation budgets.
\subsection{Organization}

Section~\ref{sec:methodology} formulates affine conditional diffusion and
derives the normal--tangent decomposition of the conditional reverse dynamics.
Section~\ref{sec:algorithm} presents the Langevin--diffusion sampler.
Section~\ref{sec:experiments} reports experiments on inpainting and
super-resolution. Section~\ref{sec:theory} gives the KL bounds and total error
decomposition. Proofs and additional derivations are deferred to the appendix.
\section{Methodology}\label{sec:methodology}

We adopt the variance-exploding (VE)  diffusion framework of \cite{song2021scorebased}. Let the clean data be
\(Z\in\R^d\) with prior law \(Z\sim p_0\). For diffusion time \(t\in[0,T]\), the forward process is
\begin{equation}\label{eq:forward_ve}
X_t := Z + W_t,
\end{equation}
where \(\{W_t\}_{t\ge 0}\) is standard Brownian motion in \(\R^d\) independent of \(Z\). Hence
\begin{equation}\label{eq:forward_kernel}
X_t \mid Z \sim \Nc(Z,\,tI_d),
\end{equation}
and we write \(p_t\) for the marginal density of \(X_t\), with score \(s_t(x):=\nabla_x\log p_t(x)\).

The time-reversal of \eqref{eq:forward_ve} yields the reverse-time stochastic
differential equation (SDE) that generates \emph{unconditional} samples
from \(p_0\). Using the reverse-time parameter \(\tau:=T-t\in[0,T]\), this SDE can be written as
\begin{equation}\label{eq:reverse_ve}
\dd Y_\tau
=
s_{T-\tau}(Y_\tau)\,\dd\tau
+
\dd \bar W_\tau,
\qquad
Y_0\sim p_T,
\end{equation}
where \(\bar W_\tau\) is a Brownian motion in reverse time. In practice, \(s_t(x)\) is approximated by a
neural network trained via score matching.

In this work we do not seek unconditional samples. Instead, we aim to sample from a conditional distribution under
a linear constraint. Let \(A\in\R^{m\times d}\) have full row rank and consider the affine constraint \(AZ=y\).
It is convenient to express the constraint through orthogonal projection onto the row space of \(A\). Define
\[
P_{\perp}:=A^\top(AA^\top)^{-1}A,
\qquad
P_{\parallel}:=I-P_{\perp},
\]
so that \(P_\perp\) projects onto \(\mathrm{range}(A^\top)\) (normal space) and \(P_\parallel\) onto \(\ker(A)\)
(tangent space). We encode the observation via the \emph{level}
\[
B := P_{\perp}Z,
\qquad
b := A^\top(AA^\top)^{-1}y,
\]
so that \(AZ=y\) is equivalent to \(B=b\), i.e., \(Z\) lies on the affine set
\[
\Mc(b):=\{x\in\R^d:\;P_\perp x=b\}.
\]
(When \(Z\) is supported on a countable codebook \(\Cc\subset\R^d\), the level \(B\) is supported on
\(P_\perp\Cc\); the development below does not otherwise rely on discreteness.)

Fix \(b\) and let \(p_t^{*,b}\) denote the conditional density of \(X_t\) under \(\Law(\,\cdot\,\mid P_{\perp}Z =b)\), with
conditional score \(s_t^{*,b}(x):=\nabla_x\log p_t^{*,b}(x)\).
If \(s_t^{*,b}\) were available, then the correct reverse-time dynamics that sample from
\(\Law(Z\mid P_{\perp}Z =b)\) would be
\begin{equation}\label{eq:cond_reverse_true_base}
\dd Y_\tau^{*,b}
=
s_{T-\tau}^{*,b}(Y_\tau^{*,b})\,\dd\tau
+
\dd \bar W_\tau,
\qquad
Y_0^{*,b}\sim \Law(X_T\mid P_{\perp}Z =b).
\end{equation}
The main obstacle is that \(s_t^{*,b}\) is not directly learned by standard unconditional score training.

For Gaussian perturbations, Tweedie’s formula expresses the conditional expectation of 
$Z$ given $X_t=x$
as
\begin{equation}\label{eq:tweedie}
\E[Z\mid X_t=x] = x + t\,s_t(x).
\end{equation}
A key observation is that, under the affine conditioning \(B=b\), Tweedie's identity immediately yields a closed-form expression for the \emph{normal} component of the conditional score: applying \(P_\perp\) to \eqref{eq:tweedie} under \(\Law(\,\cdot\,\mid P_{\perp}Z =b)\) gives
\[
P_\perp \E[Z\mid X_t=x,\,B=b] \;=\; P_\perp\big(x + t\,s_t^{*,b}(x)\big).
\]
Since \(P_\perp Z=b\) holds almost surely under \(B=b\), the left-hand side equals \(b\), and therefore
\begin{equation}\label{eq:normal_score_identity}
P_{\perp}s_t^{*,b}(x)=\frac{1}{t}\,P_{\perp}(b-x).
\end{equation}
Thus only the tangent component \(P_\parallel s_t^{*,b}\) remains unknown. Using
\(s_t^{*,b}=P_\parallel s_t^{*,b}+P_\perp s_t^{*,b}\) and substituting \eqref{eq:normal_score_identity} into
\eqref{eq:cond_reverse_true_base} yields the equivalent decomposition
\begin{equation}\label{eq:cond_reverse_true_rewrite}
\dd Y_\tau^{*,b}
=
\Big(
P_{\parallel}s_{T-\tau}^{*,b}(Y_\tau^{*,b})
+
\frac{1}{T-\tau}\,P_{\perp}\big(b-Y_\tau^{*,b}\big)
\Big)\dd \tau
+
\dd \bar W_\tau.
\end{equation}
This form makes the conditioning mechanism explicit: the process is driven toward the affine set \(\Mc(b)\) by the
{\it normal drift}, while the remaining {\it tangent drift} depends on the intractable conditional score.

To obtain a practical sampler using only an unconditional score model, we keep the \emph{exact} normal drift and
approximate the unknown tangent term by the unconditional tangent score \(P_\parallel s_t\). This yields the
{\it surrogate constrained reverse SDE} 

\begin{equation}\label{eq:cond_reverse_surrogate}
\dd \hat Y_\tau^{\,b}
=
\Big(
P_{\parallel}s_{T-\tau}(\hat Y_\tau^{\,b})
+
\frac{1}{T-\tau}\,P_{\perp}\big(b-\hat Y_\tau^{\,b}\big)
\Big)\dd \tau
+
\dd \bar W_\tau,
\qquad \tau\in[0,T-t_0).
\end{equation}

 It is constrained in the sense that its normal drift explicitly forces \(P_\perp \hat Y_\tau^{\,b}\) toward the prescribed level \(b\), thereby steering the trajectory toward the affine manifold \(\Mc(b)=\{x:\,P_\perp x=b\}\), while only the tangent component evolves according to the learned (unconditional) score.
Equations \eqref{eq:cond_reverse_true_rewrite} and \eqref{eq:cond_reverse_surrogate} share the same (exact) normal
component and differ only in the tangent score: \(P_\parallel s^{*,b}\) versus \(P_\parallel s\).
In implementations, the factor \(1/(T-\tau)=1/t\) is handled by stopping the integration at a small
\(t_{0}>0\) (equivalently \(\tau_{\max}=T-t_{0}\)) and applying a final denoising step.

\begin{remark}
   The theoretical development is stated in the VE normalization
\(X_t=Z+\sqrt t\,\xi\), because this makes the normal--tangent decomposition
transparent. The same decomposition holds for the variance-preserving (VP) denoising
diffusion probabilistic model (DDPM) forward marginals used
in our experiments. Indeed, for
\[
X_t=\alpha_t Z+\sigma_t\xi,\qquad \xi\sim\mathcal N(0,I_d),
\]
conditioning on \(P_\perp Z=b\) gives
\[
P_\perp X_t\mid P_\perp Z=b
\sim
\mathcal N(\alpha_t b,\sigma_t^2P_\perp).
\]
Hence the conditional score satisfies
\[
P_\perp s_t^{*,b}(x)
=
\frac{\alpha_t b-P_\perp x}{\sigma_t^2},
\]
while the only unknown term remains the tangent component
\(P_\parallel s_t^{*,b}(x)\). Thus the VP/DDPM analogue of
\eqref{eq:cond_reverse_surrogate} uses the surrogate score
\[
\hat s_t^{\,b}(x)
=
P_\parallel s_t(x)
+
\frac{\alpha_t b-P_\perp x}{\sigma_t^2}.
\]
This is the score used by the guided DDIM implementation in the experiments.
Therefore the experiments instantiate the same normal--tangent principle as the
VE analysis, with \(t\) replaced by the VP noise variance \(\sigma_t^2\) and the
clean constraint level \(b\) replaced by its noised mean \(\alpha_t b\).

\end{remark}

We do not integrate the surrogate reverse SDE \eqref{eq:cond_reverse_surrogate} over the full reverse-time horizon
\(\tau\in[0,T]\). The surrogate replaces the true conditional tangent score \(P_\parallel s^{*,b}_{T-\tau}\) by the
unconditional term \(P_\parallel s_{T-\tau}\). If we start at \(\tau=0\) (i.e., from the highest-noise marginal),
this mismatch acts over a long interval during which the normal correction is weak because its strength scales as
\(1/(T-\tau)=1/t\). As a result, the trajectory can drift in tangent directions in a way that is inconsistent with
the target conditional law, producing a bias that accumulates before the constraint becomes dominant at smaller noise.

To limit this accumulation, we start the surrogate reverse SDE only at an intermediate noise level \(t^*\in(0,T-t_0)\),
equivalently at reverse time \(\tau^*:=T-t^*\). Intuitively, \(t^*\) is chosen so that, for the remaining reverse
interval \(\tau\in[\tau^*,T-t_0]\) (i.e., forward times \(t\in(0,t^*]\)), using \(P_\parallel s_t\) as a proxy for
\(P_\parallel s_t^{*,b}\) is acceptable, while the normal drift is already strong enough to enforce the constraint.
What remains is that we cannot initialize the reverse SDE at \(\tau^*\) from an arbitrary point: we need an initial
state that is (approximately) distributed as the correct conditional marginal \(\Law(X_{t^*}\mid P_{\perp}Z =b)\).

We construct such an initialization by combining an exact draw for the normal component with a tangent-space sampling
step. Under the conditioning \(B=b\) we have \(P_\perp Z=b\) almost surely, hence
\[
P_\perp X_{t^*}=P_\perp(Z+W_{t^*})=b+P_\perp W_{t^*},
\]
so the normal component at time \(t^*\) can be sampled explicitly by
\[
x^{\perp}_{t^*}:= b + \sqrt{t^*}\,P_\perp \xi,
\qquad \xi\sim\Nc(0,I_d),
\]
which matches the exact law of \(P_\perp X_{t^*}\mid P_{\perp}Z =b\). Conditional on this sampled normal component \(x^\perp\),
we sample a compatible tangent component by running Langevin dynamics restricted to the affine set
\(\Mc(x^\perp):=\{x:\;P_\perp x=x^\perp\}\), using only the projected (tangent) score at time \(t^*\).

The following lemma shows that restricting a density to \(\Mc(x^\perp)\) simply projects its ambient score onto the
tangent space.

\begin{lemma}\label{lem:Bgrad_equals_Pgrad_rewrite}
Let \(A\in\R^{m\times d}\) have full row rank and let \(C\in\R^{d\times(d-m)}\) have orthonormal columns spanning
\(\ker(A)\) (\(C^\top C=I_{d-m}\), \(CC^\top=P_{\parallel}\)). Fix any \(u_0\in\Mc(x^\perp)\) and parametrize the
affine set by \(x=u_0+Cz^{\parallel}\) with \(z^{\parallel}\in\R^{d-m}\). For any differentiable density
\(p:\R^d\to(0,\infty)\), define its restriction to $\Mc(x^\perp) $ by \(\pi(z^{\parallel})\propto p(u_0+Cz^{\parallel})\). Then, for all
\(z^{\parallel}\),
\[
C\,\nabla_{z^{\parallel}}\log \pi(z^{\parallel})
=
P_{\parallel}\,\nabla_x\log p(x),
\qquad x=u_0+Cz^{\parallel}.
\]
\end{lemma}
\begin{proof}
Since the proportionality constant does not depend on \(z^\parallel\), it disappears after taking
logarithms and gradients. Thus
\[
\log \pi(z^\parallel)=\log p(u_0+Cz^\parallel)+\text{const}.
\]
Differentiating with respect to \(z^\parallel\) and using the chain rule gives
\[
\nabla_{z^\parallel}\log \pi(z^\parallel)
=
C^\top \nabla_x \log p(x),
\qquad x=u_0+Cz^\parallel.
\]
Multiplying both sides by \(C\), we obtain
\[
C\,\nabla_{z^\parallel}\log \pi(z^\parallel)
=
CC^\top \nabla_x\log p(x).
\]
Because the columns of \(C\) form an orthonormal basis of \(\ker(A)\), we have
\[
CC^\top=P_\parallel.
\]
Therefore
\[
C\,\nabla_{z^\parallel}\log \pi(z^\parallel)
=
P_\parallel \nabla_x\log p(x),
\qquad x=u_0+Cz^\parallel,
\]
which is exactly the claimed identity.
\end{proof}
We use Lemma \ref{lem:Bgrad_equals_Pgrad_rewrite} with the time-\(t^*\) marginal \(p_{t^*}\) (and its learned score
\(s_{t^*}=\nabla\log p_{t^*}\)). Starting from any point on \(\Mc(x^\perp_{t^*})\), e.g.
\[
y_0 := x^\perp_{t^*} + \sqrt{t^*}\,P_\parallel \xi,\qquad \xi\sim\Nc(0,I_d),
\]
we run underdamped Langevin dynamics evolving only in tangent directions:
\begin{equation}\label{eq:underdamped_langevin_rewrite}
\begin{aligned}
\dd y_s &= v_s\,\dd s,\\
\dd v_s &=
P_{\parallel}s_{t^*}(y_s)\,\dd s
-\gamma P_{\parallel}v_s\,\dd s
+\sqrt{2\gamma}\,P_{\parallel}\,\dd W_s,
\end{aligned}
\end{equation}
while enforcing the constraint \(P_\perp y_s\equiv x^\perp_{t^*}\) for all \(s\) (equivalently, we project updates onto the
tangent space). Let \(\hat Y_{\tau^*}^{\,b}\) denote the resulting position \(y_s\) after a prescribed Langevin time.

This two-stage procedure induces an initialization law at time \(t^*\) that matches the conditional normal marginal
exactly and uses a tractable surrogate for the tangent conditional, namely
\[
\hat p_{t^*}^{\,b}(x^{\perp},x^{\parallel})
=
p_{t^*}(x^{\parallel}\mid x^{\perp})\,p_{t^*}(x^{\perp}\mid P_{\perp}Z =b),
\qquad
x^{\perp}=P_{\perp}x,\ \ x^{\parallel}=P_{\parallel}x.
\]
Here \(p_{t^*}(x^\perp\mid P_{\perp}Z =b)\) is available in closed form because, under \(B=b\), the forward process
satisfies \(X_{t^*}^\perp=b+W_{t^*}^\perp\), hence \(X_{t^*}^\perp\sim\Nc(b,t^*P_\perp)\). The remaining factor
\(p_{t^*}(x^\parallel\mid x^\perp)\) is \emph{not} conditioned on \(B=b\); it is the \emph{unconditional} tangent
conditional induced by the pretrained model at noise level \(t^*\). Equivalently, \(\hat p_{t^*}^{\,b}\) is the
distribution obtained by (i) drawing the correct noisy normal component under the constraint, and then (ii) drawing a
tangent component that is compatible with that normal slice according to the unconditional time-\(t^*\) marginal.
This is exactly what the projected Langevin phase targets: it mixes along the affine set \(\Mc(x^\perp_{t^*})\) using the
projected score \(P_\parallel s_{t^*}\), which is the score of \(p_{t^*}(\cdot\mid x^\perp)\) restricted to the
manifold (Lemma~\ref{lem:Bgrad_equals_Pgrad_rewrite}). Finally, we use \(\hat p_{t^*}^{\,b}\) as the \emph{initial
distribution} for the surrogate reverse dynamics \eqref{eq:cond_reverse_surrogate} at reverse time
\(\tau^*=T-t^*\), i.e., \(\hat Y_{\tau^*}^{\,b}\sim \hat p_{t^*}^{\,b}\).

\begin{remark}
A common ``projection-based guidance'' heuristic enforces the affine constraint only through the analytic normal drift
while leaving the tangent component \(P_{\parallel}s_t(x)\) equal to the \emph{unconditional} tangent score.
This indeed corrects deviations from \(\Mc(b)\) in the normal directions, but it does not prevent systematic drift
\emph{along} the manifold. The issue is most severe at high noise 
the marginal \(p_t\) is a heavy Gaussian smoothing of \(p_0\), so the unconditional score \(s_t\) aggregates directions
from all modes of the data set. When the constraint \(B=b\) selects a low-probability portion of the data manifold,
the unconditional tangent drift can point toward dominant modes that are irrelevant to the observation. Because the
normal correction scales as \(1/t\), it is weakest exactly in this regime, allowing tangent errors to accumulate over a
long reverse-time horizon and leading to biased conditional samples.

This is the motivation for (i) not integrating the surrogate dynamics from \(\tau=0\), and (ii) inserting a short
projected Langevin phase at a ``safe'' noise level \(t^*\) (equivalently \(\tau^*=T-t^*\)). Starting at \(\tau^*\)
limits how long the tangent mismatch can act. The projected Langevin step then provides an approximate draw from the
needed conditional marginal \(\Law(X_{t^*}\mid P_{\perp}Z =b)\), so that the subsequent surrogate reverse dynamics begin from a
state that is already well-mixed along \(\ker(A)\) while remaining consistent with the constraint in the normal space.
\end{remark}

\medskip
\noindent\emph{Toy mixture illustration.}
We illustrate the tangent-bias mechanism on a simple prior in \(\R^2\): a three-point mixture with atoms at
\((1,1)\), \((-1,-1)\), and \((0,5)\) with weights \(0.125:0.125:0.75\).
We consider the linear constraint
\[
x-y=0,
\qquad\text{equivalently}\qquad
AZ=0\ \ \text{with}\ \ A=\begin{bmatrix}1 & -1\end{bmatrix},
\]
so the conditional target is \(\Law(Z\mid x-y=0)\), i.e., sampling on the diagonal affine set
\(\Mc(0)=\{(x,y)\in\R^2:\ x=y\}\).

We run the probability-flow ordinary differential equation (PF-ODE), the
deterministic counterpart of the reverse-time sampler, from
\(\sigma_{\max}=20\) to \(\sigma_{\min}=0.01\), with the identification \(t=\sigma^2\).
As shown in Figure~\ref{fig:toy_mixture}(a), the unconstrained PF-ODE recovers the correct mixture.

Under naive projection-based guidance initialized at \(\sigma_{\max}\), the constraint \(x-y=0\) is enforced only
through the analytic normal drift, while the tangent drift remains that of the unconditional score.
At high noise, the unconditional score is dominated by the heavy \((0,5)\) component, and this dominant-mode tangent
direction accumulates along the manifold \(\Mc(0)\), distorting the conditional weights and smearing the low-mass modes
toward the dominant cluster (Figure~\ref{fig:toy_mixture}(b)).
In contrast, our two-stage procedure runs a brief projected underdamped Langevin phase at \(t^*=0.25\) restricted to
\(\ker(A)\) (i.e., motion tangent to \(x-y=0\), cf.~Lemma~\ref{lem:Bgrad_equals_Pgrad_rewrite}),
producing an initialization close to \(\Law(X_{t^*}\mid x-y=0)\).
Starting the surrogate reverse dynamics from \(\tau^*=T-t^*\) then yields samples that remain consistent with the
constraint and recover the intended mode structure (Figure~\ref{fig:toy_mixture}(c)).

\begin{figure}[ht]
  \centering
  \begin{minipage}{0.32\textwidth}
    \centering
    \includegraphics[width=\linewidth]{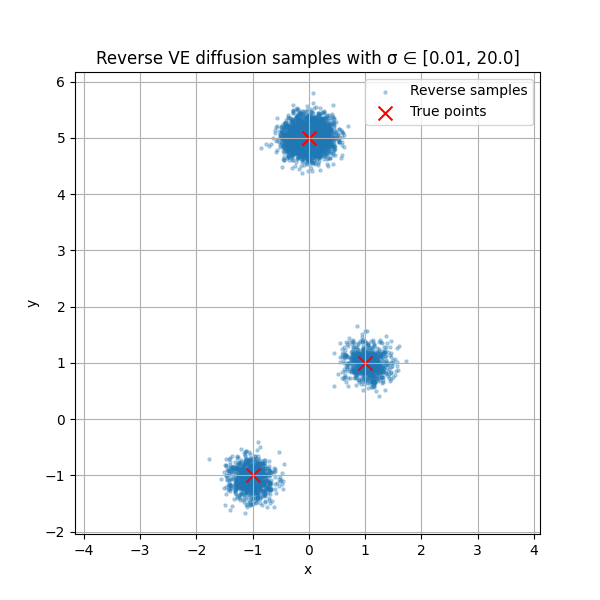}
    \centerline{\small (a) Unconstrained PF-ODE}
  \end{minipage}\hfill
  \begin{minipage}{0.32\textwidth}
    \centering
    \includegraphics[width=\linewidth]{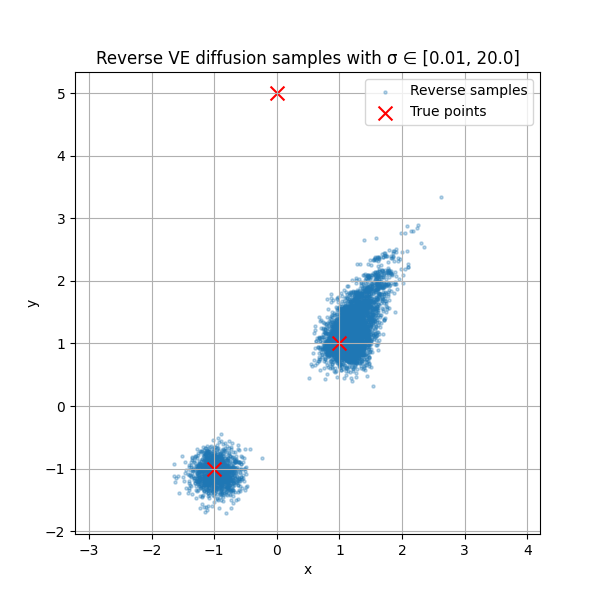}
    \centerline{\small (b) Naive projection guidance}
  \end{minipage}\hfill
  \begin{minipage}{0.32\textwidth}
    \centering
    \includegraphics[width=\linewidth]{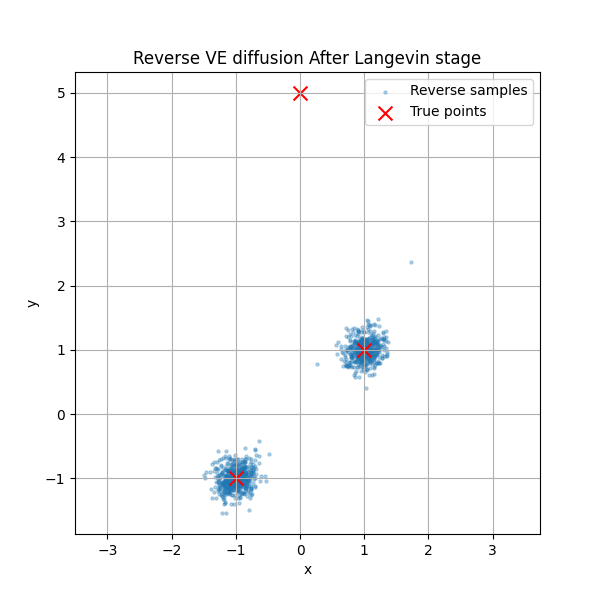}
    \centerline{\small (c) Two-stage (ours)}
  \end{minipage}
  \caption{Three-point mixture in \(\R^2\). (a) Unconstrained PF-ODE reproduces the prior mixture.
  (b) Naive projection-based guidance accumulates tangent drift dominated by the high-mass \((0,5)\) mode, biasing the
  conditional outcome. (c) Our projected Langevin initialization at \(t^*\) followed by the surrogate reverse dynamics
  recovers constraint-consistent samples with the correct mode structure.}
  \label{fig:toy_mixture}
\end{figure}

\section{Algorithm and Implementation}\label{sec:algorithm}

Our conditional sampler is organized into three conceptually distinct steps. The design goal is to (i) initialize at an
intermediate ``safe'' noise level \(t^*\), (ii) mix efficiently \emph{along} the affine constraint manifold, and (iii)
complete denoising while enforcing the constraint through an exact normal drift.
Figure~\ref{fig:diffusion-process-overview} provides a visual summary of the full pipeline. In Step~1 we move to the intermediate time \(t^*\) and fix the \emph{noisy} normal level so that
\(P_\perp X_{t^*}\) has the exact conditional law under \(B=b\). In Step~2 we run a short phase of projected
underdamped Langevin dynamics (BAOAB) on the corresponding affine set \(\Mc(x^\perp)\) to mix \emph{in the tangent
directions} while keeping the normal component fixed. This produces an initialization \(\hat Y_{\tau^*}^{\,b}\) that
is approximately distributed as \(\Law(X_{t^*}\mid P_{\perp}Z =b)\) and is already well-mixed along \(\ker(A)\), which reduces
the accumulation of tangent-score mismatch at high noise. In Step~3 we integrate the surrogate guided reverse dynamics
from \(\tau^*=T-t^*\) to \(T\), using the analytic normal drift to enforce the constraint and the pretrained score for
the tangent drift during denoising.

\emph{Step 1: Initialization for Langevin.}
As in Section~\ref{sec:methodology}, we start the reverse-time procedure at the intermediate ``safe'' noise level
\(t^*=T-\tau^*\). Step~1 produces the \emph{initial state for the projected Langevin phase} (Step~2).
To do so, we select any clean feasible point \(x_0\in\Mc(b)\) satisfying \(P_\perp x_0=b\) (equivalently \(Ax_0=y\)).
The choice of \(x_0\) is not unique and does not affect feasibility; in practice one may take, for example,
\(x_0=b+P_\parallel \zeta\) with \(\zeta\sim\Nc(0,I_d)\) (Gaussian initialization in \(\ker(A)\)), or use a plug-in
estimate (e.g., a pseudoinverse or any other fast reconstruction) and project it onto \(\Mc(b)\).

We then move \(x_0\) to time \(t^*\) by adding Gaussian perturbation,
\[
y^{\tau^*} \;=\; x_0 + \sqrt{T-\tau^*}\,\xi \;=\; x_0+\sqrt{t^*}\,\xi,
\qquad \xi\sim\Nc(0,I_d).
\]
Rather than enforcing the clean constraint level \(b\) at this stage, we freeze the \emph{noisy} normal component
\[
b_{\mathrm{noisy}} \;:=\; P_\perp y^{\tau^*},
\]
which is the correct forward-time stochastic normal level at time \(t^*\) under the conditioning \(B=b\).
\(y^{\tau^*}\) is then used to initialize the constrained BAOAB/underdamped Langevin
dynamics on the affine set \(\Mc(b_{\mathrm{noisy}})\) in Step~2.\\
\emph{Step 2: Tangent BAOAB Langevin}
Next, we approximate the conditional marginal \(\Law(X_{t^*}\mid P_{\perp}Z =b)\) by running \emph{underdamped Langevin dynamics}
restricted to the affine set $\Mc(b_{\text{noisy}})$.
We evolve a position--velocity pair \((y_s,v_s)\) using the projected dynamics in Equation \eqref{eq:underdamped_langevin_rewrite}, 
so that both the deterministic ``force'' and the stochastic excitation act \emph{only in tangent directions}
\(\ker(A)\). We discretize this SDE with the \textbf{BAOAB splitting integrator}, which decomposes the dynamics into
three sub-operators that can be integrated in closed form.

\smallskip
\noindent\textbf{The BAOAB split.}
Write the SDE as the sum of:
\begin{itemize}
  \item \emph{B (kick):} deterministic velocity update due to the force
  \(\;\dot v = P_{\parallel}s_{t^*}(y)\).
  \item \emph{A (drift):} deterministic position update
  \(\;\dot y = v\).
  \item \emph{O (Ornstein--Uhlenbeck):} stochastic friction/noise on velocity
  \(\;\dd v = -\gamma P_{\parallel}v\,\dd s + \sqrt{2\gamma}\,P_{\parallel}\dd W_s\).
\end{itemize}
BAOAB applies these pieces in the symmetric order
\[
\text{B/2} \;\rightarrow\; \text{A/2} \;\rightarrow\; \text{O} \;\rightarrow\; \text{A/2} \;\rightarrow\; \text{B/2},
\]
which is time-reversible (in the deterministic limit) and is known to have excellent stability and low bias in the
\emph{configurational} (position) marginal.

\smallskip
With step size \(\Delta s\), one iteration from \((y,v)\) proceeds as:
\begin{enumerate}
  \item \emph{B/2 (half kick):} update the velocity using the score force at the current position:
  \[
  v \leftarrow v + \frac{\Delta s}{2}\,P_{\parallel}s_{t^*}(y).
  \]
  \item \emph{A/2 (half drift):} move the position forward using the current velocity:
  \[
  y \leftarrow y + \frac{\Delta s}{2}\,v.
  \]
  \item \emph{O (OU refresh):} apply friction and inject Gaussian noise directly in velocity.
  This step is exact because it is an Ornstein--Uhlenbeck process. Writing
  \[
  c_1 := e^{-\gamma \Delta s},
  \qquad
  c_2 := \sqrt{\big(1-e^{-2\gamma\Delta s}\big)},
  \]
  we perform
  \[
  v \leftarrow c_1 v + c_2\,P_{\parallel}\xi,
  \qquad \xi\sim\Nc(0,I_d).
  \]
  Here \(c_1\) contracts velocity (friction) and \(c_2\) sets the noise amplitude; the projection \(P_{\parallel}\)
  ensures that the OU excitation does not change the normal component.
  \item \emph{A/2 (half drift):} advance the position again:
  \[
  y \leftarrow y + \frac{\Delta s}{2}\,v.
  \]
  \item \emph{B/2 (half kick):} apply the remaining half force update:
  \[
  v \leftarrow v + \frac{\Delta s}{2}\,P_{\parallel}s_{t^*}(y).
  \]
\end{enumerate}

\smallskip
 After \(K\) BAOAB iterations, we denote the resulting
position by \(\hat Y_{\tau^*}^b\). This state is well-mixed along \(\ker(A)\) while remaining consistent with the
forward-time noisy level set, making it a reliable initialization for the guided reverse denoising stage.
\\
\emph{Step 3: Guided Reverse Denoising}
Finally, starting from \(\hat Y_{\tau^*}^b\) we integrate the guided reverse SDE from \(\tau =\tau^*\) up to $T-t_0$ in Equation \eqref{eq:cond_reverse_surrogate}

\medskip
Algorithm~\ref{alg:baoab_sampling} summarizes these three steps in pseudocode.

\begin{algorithm}
\caption{Conditional Sampling via Affine BAOAB Initialization}
\label{alg:baoab_sampling}
\begin{algorithmic}[1]
\STATE {\bfseries Input:} Clean measurement $b$, starting point $x_0 \in \mathcal{M}(b)$, intermediate noise $t^* = T - \tau^*$, Langevin steps $K$, step size $\Delta s$, friction $\gamma$, score network $s_{t^*}$.
\STATE {\bfseries Step 1: Initialization for Langevin}
\STATE $y^{\tau^*} \leftarrow x_0 + \sqrt{T - \tau^*}\xi, \quad \xi \sim \mathcal{N}(0, I_d)$
\STATE $b_{\text{noisy}} \leftarrow P_{\perp} y^{\tau^*}$ \COMMENT{Target level set for the Langevin phase}
\STATE $y \leftarrow y^{\tau^*}, \quad v \leftarrow 0$
\STATE $c_1 \leftarrow e^{-\gamma \Delta s}, \quad c_2 \leftarrow \sqrt{1 - e^{-2\gamma \Delta s}}$

\STATE {\bfseries Step 2: Tangent BAOAB Langevin}
\FOR{$k = 1$ {\bfseries to} $K$}
    \STATE $v \leftarrow v + \frac{\Delta s}{2} P_{\parallel} s_{t^*}(y)$ \hfill \COMMENT{\textbf{B}: Half-step drift}
    \STATE $y \leftarrow y + \frac{\Delta s}{2} v$ \hfill \COMMENT{\textbf{A}: Half-step position}
    \STATE $v \leftarrow c_1 v + c_2 P_{\parallel} \xi, \quad \xi \sim \mathcal{N}(0, I_d)$ \hfill \COMMENT{\textbf{O}: Projected noise injection}
    \STATE $y \leftarrow y + \frac{\Delta s}{2} v$ \hfill \COMMENT{\textbf{A}: Half-step position}
    \STATE $v \leftarrow v + \frac{\Delta s}{2} P_{\parallel} s_{t^*}(y)$ \hfill \COMMENT{\textbf{B}: Half-step drift}
    \STATE $y \leftarrow P_{\parallel} y + b_{\text{noisy}}$ \hfill \COMMENT{Constraint: Maintain $P_{\perp} y = b_{\text{noisy}}$}
\ENDFOR
\STATE $\hat{Y}_{\tau^*}^b \leftarrow y$

\STATE {\bfseries Step 3: Guided Reverse Denoising}
\STATE Evolve $\hat{Y}_{\tau^*}^b$ from $\tau = \tau^*$ to $T-t_0$ using the guided reverse SDE in Equation \eqref{eq:cond_reverse_surrogate}
\STATE {\bfseries Return:} Final conditional sample $\hat{z}$
\end{algorithmic}
\end{algorithm}
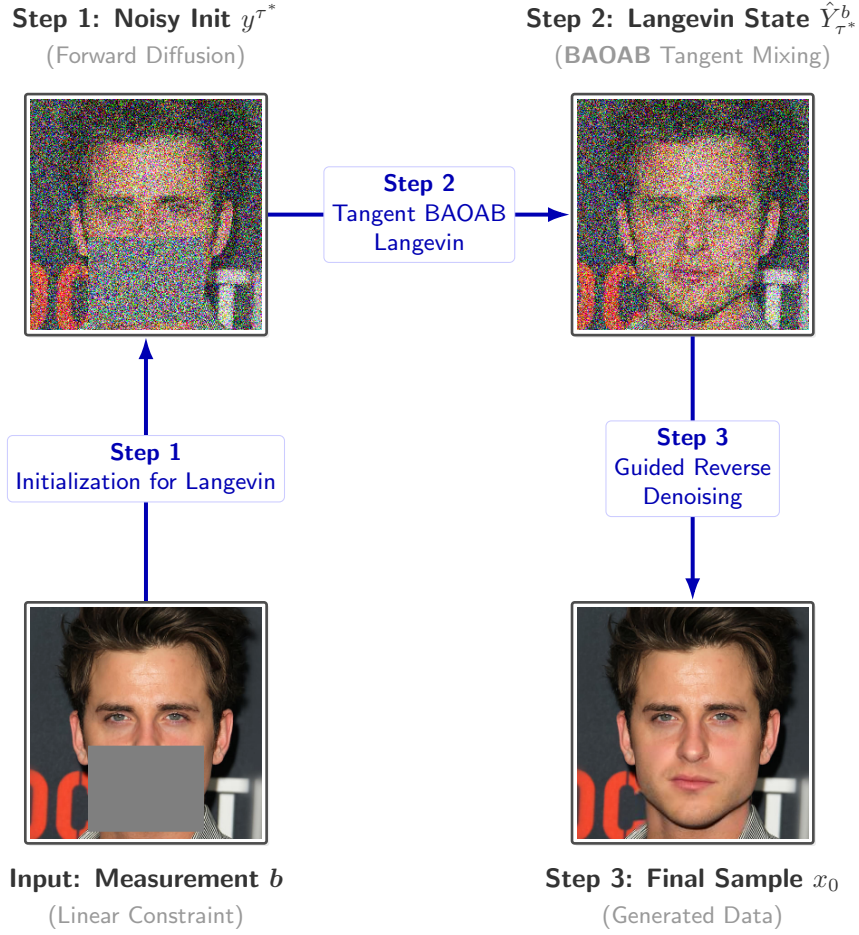
\begin{figure}[ht]
  \centering
  \tikzset{
    imgnode/.style={
        rectangle,
        draw=gray!60!black,
        line width=1pt,
        inner sep=2pt,
        rounded corners=1pt,
        fill=white,
        drop shadow={opacity=0.15, shadow xshift=1pt, shadow yshift=-1pt}
    },
    flowarrow/.style={
        ->,
        >={Latex[length=3mm, width=2mm]},
        line width=1.5pt,
        color=blue!70!black
    },
    flowlabel/.style={
        midway,
        fill=white,
        align=center,
        inner sep=3pt,
        rounded corners=2pt,
        draw=blue!20!white,
        thin,
        font=\sffamily\footnotesize
    }
  }

  \begin{tikzpicture}[node distance=3.5cm and 4cm]
    \node[imgnode] (n1) {
        \includegraphics[width=0.2\textwidth]{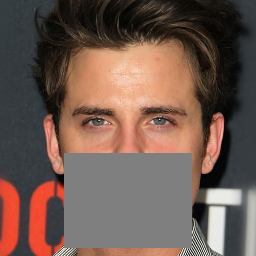}
    };
    \node[below=5pt of n1, align=center] {
        \mainlabel{Input: Measurement $\bm{b}$}\\
        \sublabel{(Linear Constraint)}
    };

    \node[imgnode] (n2) [above=of n1] {
        \includegraphics[width=0.2\textwidth]{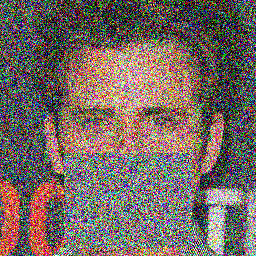}
    };
    \node[above=5pt of n2, align=center] {
        \mainlabel{\textbf{Step 1:} Noisy Init $y^{\tau^*}$}\\
        \sublabel{(Forward Diffusion)}
    };

    \node[imgnode] (n3) [right=of n2] {
        \includegraphics[width=0.2\textwidth]{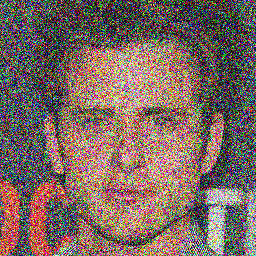}
    };
    \node[above=5pt of n3, align=center] {
        \mainlabel{\textbf{Step 2:} Langevin State $\hat{Y}_{\tau^*}^b$}\\
        \sublabel{(\textbf{BAOAB} Tangent Mixing)}
    };

    \node[imgnode] (n4) [right=of n1] {
        \includegraphics[width=0.2\textwidth]{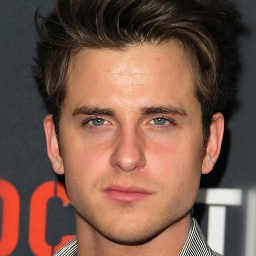}
    };
    \node[below=5pt of n4, align=center] {
        \mainlabel{\textbf{Step 3:} Final Sample $x_0$}\\
        \sublabel{(Generated Data)}
    };

    \draw[flowarrow] (n1.north) -- node[flowlabel] {\textbf{Step 1}\\Initialization for Langevin} (n2.south);

    \draw[flowarrow] (n2.east) -- node[flowlabel]
      {\textbf{Step 2}\\Tangent BAOAB \\ Langevin} (n3.west);

    \draw[flowarrow] (n3.south) -- node[flowlabel] {\textbf{Step 3}\\Guided Reverse \\ Denoising} (n4.north);

  \end{tikzpicture}
  \caption{Visual overview of the proposed sampling process. Step 1: diffuse the constrained input to the intermediate noise level \(t^*\). Step 2: run projected BAOAB underdamped Langevin dynamics to mix along the affine constraint set while preserving the noisy normal level. Step 3: perform guided reverse denoising with exact normal correction to obtain the final sample.}
  \label{fig:diffusion-process-overview}
\end{figure}

\section{Experiments}
\label{sec:experiments}

We evaluate the proposed Langevin-Conditioned Diffusion Model with BAOAB (\textbf{LCDM-BAOAB}) sampler on standard
\(256\times256\) image inverse problems. LCDM-BAOAB uses the affine
normal--tangent decomposition developed in the previous sections: it first
performs projected BAOAB Langevin mixing in the tangent directions at an
intermediate noise level, and then completes sampling by guided DDIM denoising
with exact normal correction.

We test on three benchmarks: CelebA-HQ \citep{karras2017progressive},
LSUN Church \citep{yu2015lsun}, and ImageNet \citep{deng2009imagenet}. As the
primary baseline, we use the 
DDNM \citep{wang2022ddnm}, a strong zero-shot diffusion method for linear
image inverse problems.

We compare with DDNM because it has been reported to meaningfully outperform
earlier zero-shot conditional sampling and restoration methods for linear
inverse problems. In particular, DDNM was introduced as a unified zero-shot
framework for linear image restoration tasks such as super-resolution,
inpainting, colorization, compressed sensing, and deblurring, and was shown to
improve over prior zero-shot approaches including ILVR, RePaint, DDRM, and DPS
\citep{choi2021ilvr,lugmayr2022repaint,kawar2022ddrm,wang2022ddnm, chung2023dps}. Thus, DDNM
provides a strong projection-based reference point for testing whether the
additional tangent-space BAOAB Langevin initialization in LCDM-BAOAB yields
measurable improvements under matched compute.

In Appendix~\ref{app:vp-ddpm-normal-correction}, we show that, under the VP--DDPM
\(\varepsilon\)-parameterization, the DDNM update is equivalent to using the
effective score
\[
\hat s_t^{\rm DDNM}(x_t;y)
=
P_\parallel s_t(x_t)
+
\frac{\alpha_t b-P_\perp x_t}{\sigma_t^2}.
\]
Thus DDNM applies the analytic correction in the normal directions while
retaining the pretrained unconditional score in the tangent directions. Our
experiments therefore test whether explicitly mixing in the tangent directions
through projected BAOAB Langevin dynamics improves over a strong zero-shot
projection-based sampler.

All experiments are performed in the zero-shot setting using pretrained
\(256\times256\) diffusion backbones, with no task-specific fine-tuning. We
report LPIPS and FID; lower values are better for both metrics.

All experiments are run under a matched budget of \(100\) effective network
function evaluations (NFEs). For DDNM, this corresponds to \(100\) DDIM steps
with \(\eta=0.85\). For LCDM-BAOAB, the budget is split into \(50\) projected
BAOAB Langevin updates and \(50\) guided DDIM denoising steps. In the
\(8\times\) super-resolution experiments, we use a \(200\)-point DDIM time grid
and start LCDM-BAOAB from the \(25\%\) point of this grid, so the guided reverse
stage uses the final \(50\) DDIM network evaluations. Thus all reported
comparisons use the same total budget of \(100\) NFEs. In the BAOAB phase, we cache and reuse the previous
UNet output whenever possible, so that the effective number of score-network
evaluations remains matched to DDNM.

The Langevin phase is introduced at a task-dependent discrete DDPM timestep,
denoted \(k_{\mathrm{mix}}\). This index refers to the implementation timestep
of the pretrained DDPM sampler and should not be confused with the continuous
safe-time parameter \(t^*\) used in the theoretical analysis. For inpainting, we
use \(k_{\mathrm{mix}}=500\). For super-resolution, we use
\(k_{\mathrm{mix}}=250\), which corresponds to a later and higher-SNR point in
the reverse trajectory. This choice reflects the different nature of the two
inverse problems. In super-resolution, the main difficulty is recovering
high-frequency detail from a heavily downsampled image, and tangent-space
refinement is more stable once the iterate is closer to the data manifold.
Therefore, for super-resolution, we perform BAOAB mixing later in the denoising
trajectory than we do for masking tasks.

For super-resolution, we consider \(8\times\) mean downsampling, mapping
\(32\times32\) observations to \(256\times256\) images. Quantitative results are
reported on \(1000\) images per data set.

To facilitate reproducibility, we provide the implementation, configuration files,
and experiment scripts at
\url{https://github.com/ahmad-aghapour/lcdm}.
\subsection{Inpainting Results}

We first evaluate fixed-mask inpainting. The fixed mask is chosen differently
across data sets. On CelebA-HQ, we mask a facial region, since reconstructing a
semantically important part of a human face is substantially more challenging
than filling an arbitrary patch. On LSUN Church and ImageNet, we use the
corresponding fixed square masks for those data sets. This distinction is
important when interpreting the CelebA-HQ results, since the CelebA-HQ mask
targets a harder semantic completion problem.

Table~\ref{tab:inpainting_all} shows that LCDM-BAOAB consistently improves over
DDNM on all three data sets. The improvement is modest on CelebA-HQ, but becomes
larger on LSUN Church and especially on ImageNet. This trend is consistent with
our hypothesis: as the data distribution becomes more diverse and the tangent
space becomes more semantically ambiguous, projection-only guidance is more
susceptible to tangent-space bias, and explicit tangent mixing becomes more
beneficial.

\begin{table}[ht]
\centering
\begin{tabular}{lcccccc}
\toprule
\textbf{Method}
& \multicolumn{2}{c}{\textbf{CelebA-HQ}}
& \multicolumn{2}{c}{\textbf{LSUN Church}}
& \multicolumn{2}{c}{\textbf{ImageNet}} \\
\cmidrule(lr){2-3} \cmidrule(lr){4-5} \cmidrule(lr){6-7}
& \textbf{LPIPS} & \textbf{FID}
& \textbf{LPIPS} & \textbf{FID}
& \textbf{LPIPS} & \textbf{FID} \\
\midrule
DDNM
& 0.0579 & 13.82
& 0.1156 & 11.51
& 0.1242 & 26.13 \\
\textbf{LCDM-BAOAB (ours)}
& \textbf{0.0421} & \textbf{11.33}
& \textbf{0.0973} & \textbf{9.63}
& \textbf{0.0985} & \textbf{18.35} \\
\bottomrule
\end{tabular}
\caption{Fixed-mask inpainting results on \(256\times256\) benchmarks. Metrics are LPIPS\(\downarrow\) and FID\(\downarrow\).}\label{tab:inpainting_all}
\end{table}

To test whether the inpainting improvement persists beyond a single fixed
corruption pattern, we also evaluate random-mask inpainting on \(1000\) images
per data set. For each image, we remove a \(100\times100\) square patch sampled
at a random location. The random mask is tied deterministically to the image
identity, so DDNM and LCDM-BAOAB are evaluated on exactly the same corrupted
input for each image. This gives a paired comparison and removes any ambiguity
about whether differences are caused by the sampler or by different mask
locations.

Table~\ref{tab:randommask_inpainting} shows that LCDM-BAOAB again improves over
DDNM across all three data sets. The gains are smallest on CelebA-HQ, larger on
LSUN Church, and largest on ImageNet. On ImageNet, LCDM-BAOAB improves FID from
\(29.00\) to \(20.91\) and LPIPS from \(0.1182\) to \(0.0933\). These results
show that tangent-space BAOAB mixing remains beneficial even when the missing
region varies across images.

\begin{table}[ht]
\centering
\begin{tabular}{lcccccc}
\toprule
\textbf{Method}
& \multicolumn{2}{c}{\textbf{CelebA-HQ}}
& \multicolumn{2}{c}{\textbf{LSUN Church}}
& \multicolumn{2}{c}{\textbf{ImageNet}} \\
\cmidrule(lr){2-3} \cmidrule(lr){4-5} \cmidrule(lr){6-7}
& \textbf{LPIPS} & \textbf{FID}
& \textbf{LPIPS} & \textbf{FID}
& \textbf{LPIPS} & \textbf{FID} \\
\midrule
DDNM
& 0.0499 & 7.23
& 0.1102 & 11.84
& 0.1182 & 29.00 \\
\textbf{LCDM-BAOAB (ours)}
& \textbf{0.0406} & \textbf{6.23}
& \textbf{0.0936} & \textbf{9.55}
& \textbf{0.0933} & \textbf{20.91} \\
\bottomrule
\end{tabular}
\caption{Random-mask inpainting results on \(1000\) images per data set. For each image, a \(100\times100\) square mask is sampled at a random location and shared across methods. Metrics are LPIPS\(\downarrow\) and FID\(\downarrow\).}
\label{tab:randommask_inpainting}
\end{table}

\begin{figure}[ht]
\centering
\setlength{\tabcolsep}{2pt}
\renewcommand{\arraystretch}{0.5}
\begin{tabular}{ccc}
\includegraphics[width=0.31\linewidth]{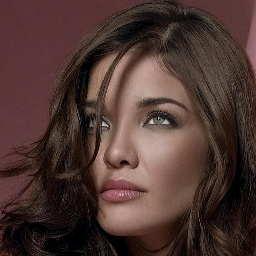} &
\includegraphics[width=0.31\linewidth]{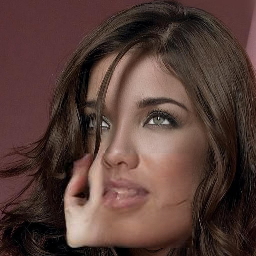} &
\includegraphics[width=0.31\linewidth]{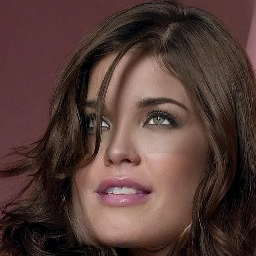} \\
(a) Real image & (b) DDNM & (c) LCDM-BAOAB \\[4pt]
\includegraphics[width=0.31\linewidth]{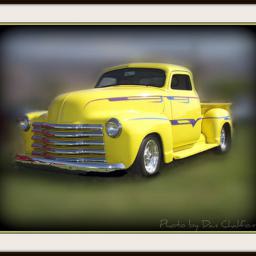} &
\includegraphics[width=0.31\linewidth]{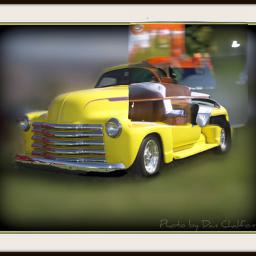} &
\includegraphics[width=0.31\linewidth]{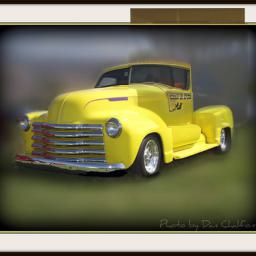} \\
(d) Real image & (e) DDNM & (f) LCDM-BAOAB
\end{tabular}
\caption{Visual comparison for inpainting. DDNM can produce texture artifacts or semantically inconsistent completions, especially on ImageNet. LCDM-BAOAB produces cleaner and more coherent reconstructions while preserving measurement consistency.}
\label{fig:tangent-bias}
\end{figure}

\subsection{Super-Resolution Results}

We next evaluate \(8\times\) super-resolution, where the observation is obtained
by mean downsampling a \(256\times256\) image to \(32\times32\). This inverse
problem is substantially more ill-posed than inpainting because most
high-frequency information is removed by the forward operator. Consequently,
the conditional distribution contains a large tangent-space ambiguity: many
high-resolution images are consistent with the same low-resolution observation.

For this task, we introduce the BAOAB Langevin phase at the discrete DDPM
timestep \(k_{\mathrm{mix}}=250\), later in the reverse trajectory than in the
inpainting experiments. This higher-SNR starting point makes tangent refinement
more stable and allows the sampler to recover fine-scale structure after the
coarse image content has already been established.

Table~\ref{tab:sr_results} shows that LCDM-BAOAB consistently improves over
DDNM on all three data sets. The gains are largest on ImageNet, where the
conditional ambiguity is strongest, and remain substantial on LSUN Church. These
results support the central claim of the paper: enforcing the measurement in the
normal directions is not sufficient for highly ill-posed linear inverse
problems; additional tangent-space mixing can substantially improve perceptual
and distributional quality.

\begin{table}[ht]
\centering
\begin{tabular}{lcccccc}
\toprule
\textbf{Method}
& \multicolumn{2}{c}{\textbf{CelebA-HQ}}
& \multicolumn{2}{c}{\textbf{LSUN Church}}
& \multicolumn{2}{c}{\textbf{ImageNet}} \\
\cmidrule(lr){2-3} \cmidrule(lr){4-5} \cmidrule(lr){6-7}
& \textbf{LPIPS} & \textbf{FID}
& \textbf{LPIPS} & \textbf{FID}
& \textbf{LPIPS} & \textbf{FID} \\
\midrule
DDNM
& 0.1437 & 33.88
& 0.2765 & 28.80
& 0.3242 & 60.74 \\
\textbf{LCDM-BAOAB (ours)}
& \textbf{0.1228} & \textbf{29.58}
& \textbf{0.2411} & \textbf{21.03}
& \textbf{0.2920} & \textbf{45.57} \\
\bottomrule
\end{tabular}
\caption{\(8\times\) super-resolution results on \(256\times256\) benchmarks. Metrics are LPIPS\(\downarrow\) and FID\(\downarrow\).}
\label{tab:sr_results}
\end{table}

\begin{figure}[ht]
\centering
\setlength{\tabcolsep}{2pt}
\begin{tabular}{ccc}
\includegraphics[width=0.31\linewidth]{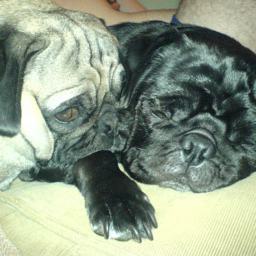} &
\includegraphics[width=0.31\linewidth]{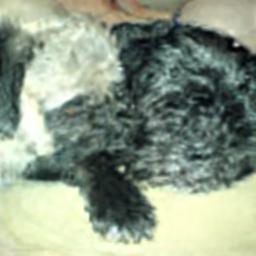} &
\includegraphics[width=0.31\linewidth]{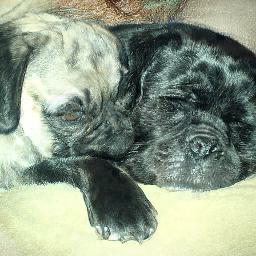} \\
(a) Real image & (b) DDNM, \(8\times\) & (c) LCDM-BAOAB, \(8\times\)
\end{tabular}
\caption{\(8\times\) super-resolution on ImageNet. DDNM tends to produce blurred or structurally inconsistent outputs, while LCDM-BAOAB recovers sharper edges and more realistic high-frequency detail.}
\label{fig:sr_visuals}
\end{figure}

\FloatBarrier
\section{Error Decomposition and Average KL Bounds}\label{sec:theory}
We now state quantitative guarantees for the discrepancy between the ideal conditional sampler and our practical
procedure. The key idea is to align the analysis with the algorithmic structure: starting from a safe noise level
\(t^*\), our method (i) approximately initializes the reverse process at time \(t^*\) using the two-stage normal/tangent
construction, and (ii) then evolves via a guided reverse SDE whose normal drift is exact but whose tangent drift uses
the unconditional score. Accordingly, the results below separate the total error into an \emph{initialization} term at
time \(t^*\) and a \emph{pathwise} term accumulated during the reverse evolution. The first theorem bounds the pathwise
KL divergence between the true conditional and surrogate reverse-time \emph{path measures} in terms of conditional mutual
information between tangent and normal components. We then introduce assumptions that control how the tangent
conditional marginal varies with the level \(b\), and use them to bound the average initialization error. Combining
these ingredients yields an average terminal KL bound for the tangent marginal of the generated sample, together with
sharper consequences under additional separation conditions on the admissible levels.
We now quantify the error of our conditional sampling procedure. There are two
conceptual contributions:

\begin{itemize}
\item[(i)] a \emph{pathwise} error from using the unconditional tangent score
in the guided reverse SDE instead of the true conditional tangent score;
\item[(ii)] an \emph{initialization} error at time $t^*$, because we only
approximately sample from the true conditional marginal $\Law(X_{t^*}\mid P_{\perp}Z =b)$ using the two-stage procedure described above.
\end{itemize}

The next theorem controls the pathwise error in terms of conditional mutual
information between tangent and normal components. To prove it, we only need a second moment bound on $Z$.

\begin{assumption}\label{assumption0}
For $Z\sim p_0$, we have $\E\|Z\|^2<\infty$.
\end{assumption}

\begin{theorem}
\label{thm:avg-path-kl}
Let Assumption~\ref{assumption0} be in force. Fix \(0\le \tau^*<T-t_0\). For each level
\(b\), consider on \([\tau^*,T-t_0]\) the ideal conditional reverse SDE
\eqref{eq:cond_reverse_true_rewrite} and the surrogate constrained reverse SDE
\eqref{eq:cond_reverse_surrogate}, started from the same initial law at time \(\tau^*\):
\[
\begin{aligned}
\dd Y_\tau^{*,b}
&=
\Big(
P_{\parallel}s_{T-\tau}^{*,b}(Y_\tau^{*,b})
+\frac{1}{T-\tau}P_{\perp}(b-Y_\tau^{*,b})
\Big)\,\dd\tau
+\dd\bar W_\tau,\\[4pt]
\dd \hat Y_\tau^{\,b}
&=
\Big(
P_{\parallel}s_{T-\tau}(\hat Y_\tau^{\,b})
+\frac{1}{T-\tau}P_{\perp}(b-\hat Y_\tau^{\,b})
\Big)\,\dd\tau
+\dd\bar W_\tau.
\end{aligned}
\]
Let \(\PP^{Y^{*,b}}\) and \(\PP^{\hat Y^{\,b}}\) denote the corresponding path measures on
\([\tau^*,T-t_0]\). Decompose the clean signal as
\[
Z^\parallel:=P_\parallel Z,
\qquad
Z^\perp:=P_\perp Z.
\]
Then
\[
\E_B\!\Big[\KL\!\big(\PP^{Y^{*,B}}\,\Vert\,\PP^{\hat Y^{\,B}}\big)\Big]
\le
I\!\big(Z^\parallel;Z^\perp\mid X_{t^*}\big).
\]
Moreover,
\[
\E_B\!\Big[\KL\!\big(\PP^{Y^{*,B}}\,\Vert\,\PP^{\hat Y^{\,B}}\big)\Big]
\ge
I\!\big(Z^\parallel;Z^\perp\mid X_{t^*}\big)
-
I\!\big(Z^\parallel;Z^\perp\mid X_{t^*}^\perp\big)
-
I\!\big(Z^\parallel;Z^\perp\mid X_{t_0}\big).
\]
\end{theorem}
Theorem~\ref{thm:avg-path-kl} is stated for a general clean signal \(Z\), and shows that the
pathwise error is controlled by the conditional mutual information
\[
I\!\big(Z^\parallel;Z^\perp\mid X_{t^*}\big).
\]
To make this quantity more concrete, we now specialize to a latent Gaussian-mixture model.
This specialization is motivated by modern discrete-latent generative models for images, in
which the observed image can be viewed as a structured latent code decoded into pixel space
up to a small reconstruction error. Under this model, \(Z\) is a Gaussian perturbation of a
discrete latent variable \(S\), and the dependence term in Theorem~\ref{thm:avg-path-kl}
can be compared to the corresponding latent dependence
\[
I\!\big(S^\parallel;S^\perp\mid X_{t^*}\big).
\]
The next assumption and proposition formalize this reduction.

\begin{assumption}
\label{ass:latent_gaussian_mixture}
The clean signal \(Z\in\R^d\) admits the representation
\[
Z=S+\varepsilon N,
\]
where \(S\) is a discrete random vector taking values in a countable set
\(\Cc\subset\R^d\), \(N\sim\Nc(0,I_d)\), \(N\) is independent of \(S\), and
\(\varepsilon>0\). Equivalently, \(Z\) follows a countable Gaussian mixture
distribution whose components have means in \(\Cc\) and common covariance
\(\varepsilon^2 I_d\).

We write
\[
S^\parallel:=P_\parallel S,\qquad S^\perp:=P_\perp S,
\]
and
\[
Z^\parallel:=P_\parallel Z,\qquad Z^\perp:=P_\perp Z.
\]
In addition, we assume that the projected latent normal code has finite entropy,
\[
H(S^\perp)<\infty .
\]
Whenever Rényi-entropy bounds are invoked, we further assume that the order-\(1/2\)
Rényi entropy is finite:
\[
H_{1/2}(S^\perp)<\infty .
\]
\end{assumption}
\begin{remark}
Assumption~\ref{ass:latent_gaussian_mixture} should be viewed as a mild structural model
rather than a restrictive finite-mixture hypothesis. It says that the observed image
\(Z\in\R^d\) can be decomposed as
\[
Z=S+\varepsilon N,\qquad N\sim\Nc(0,I_d),
\]
where \(S\) captures the semantic or low-dimensional content and \(\varepsilon N\) models
small reconstruction or modeling error.

This assumption is motivated by modern discrete-latent image models. In vector-quantized
autoencoding approaches, an image is first represented by a compact latent code and then
decoded back to pixel space. For instance, TiTok represents a \(256\times256\) image using
as few as \(32\) discrete tokens, showing that high-dimensional images can often be
described through a much smaller latent representation. 
If one interprets the decoder output as the structured component \(S\) and treats the
remaining reconstruction error in pixel space as approximately Gaussian and isotropic, then
the resulting image model takes exactly the form
\[
Z=S+\varepsilon N.
\]

We do not require \(\mathcal C\) to be finite; countably infinite support is allowed,
provided the entropy quantities appearing below are finite. Thus the assumption covers both genuinely finite codebook models and more general
latent representations in which the image is concentrated near a structured set and the
residual variability is small and approximately Gaussian.

Under this viewpoint, the conditional mutual information
\[
I\!\big(S^\parallel;S^\perp\mid X_{t^*}\big)
\]
measures the residual coupling between tangent and normal latent directions after observing
the image through an effective Gaussian channel. This makes the latent formulation
particularly well aligned with modern tokenized image representations: the theorem is not
trying to model raw pixels directly as discrete objects, but rather to exploit the fact that
pixel-space images are often well approximated by a structured latent representation plus a
small Gaussian perturbation.
\end{remark}

\begin{proposition}
\label{prop:latent_cmi_dpi}
Let Assumption~\ref{ass:latent_gaussian_mixture} be in force. Then, for every \(t\ge 0\),
\[
I\!\big(Z^\parallel;Z^\perp\mid X_t\big)
\le
I\!\big(S^\parallel;S^\perp\mid X_t\big).
\]
In particular, at the safe time \(t^*\),
\[
I\!\big(Z^\parallel;Z^\perp\mid X_{t^*}\big)
\le
I\!\big(S^\parallel;S^\perp\mid X_{t^*}\big).
\]
\end{proposition}

\begin{proof}
Fix \(t\ge0\). Under Assumption~\ref{ass:latent_gaussian_mixture},
\[
Z=S+\varepsilon N,
\qquad
X_t=Z+W_t=S+\varepsilon N+W_t,
\]
where \(N\) and \(W_t\) are independent standard Gaussian noises. Since
\(P_\parallel\) and \(P_\perp\) are orthogonal projections, the tangent and normal
noise components are independent. Hence, for every regular conditional law given
\(X_t=x\),
\[
\Law\!\big(Z^\parallel,Z^\perp \mid S^\parallel,S^\perp,X_t=x\big)
=
\Law\!\big(Z^\parallel \mid S^\parallel,X_t^\parallel=x^\parallel\big)
\otimes
\Law\!\big(Z^\perp \mid S^\perp,X_t^\perp=x^\perp\big).
\]
Thus, conditionally on \(X_t=x\), the pair
\((Z^\parallel,Z^\perp)\) is obtained from
\((S^\parallel,S^\perp)\) by applying two separate conditionally independent
channels: one from \(S^\parallel\) to \(Z^\parallel\), and one from
\(S^\perp\) to \(Z^\perp\). Therefore, by the data-processing inequality for
mutual information under product channels,
\[
I_{\Law(\cdot\mid X_t=x)}\!\big(Z^\parallel;Z^\perp\big)
\le
I_{\Law(\cdot\mid X_t=x)}\!\big(S^\parallel;S^\perp\big)
\]
for \(X_t\)-almost every \(x\). Integrating this inequality with respect to the
law of \(X_t\) gives
\[
I\!\big(Z^\parallel;Z^\perp\mid X_t\big)
\le
I\!\big(S^\parallel;S^\perp\mid X_t\big).
\]
The statement at \(t=t^*\) is the same inequality evaluated at the safe time.
\end{proof}

\begin{remark}
\label{rem:pathwise_latent_interpretation}
Theorem~\ref{thm:avg-path-kl} shows that the pathwise error incurred by replacing the
unknown conditional tangent score with the unconditional tangent score is controlled by
\[
I\!\big(Z^\parallel;Z^\perp\mid X_{t^*}\big).
\]
Under Assumption~\ref{ass:latent_gaussian_mixture}, Proposition~\ref{prop:latent_cmi_dpi}
further yields
\[
I\!\big(Z^\parallel;Z^\perp\mid X_{t^*}\big)
\;\le\;
I\!\big(S^\parallel;S^\perp\mid X_{t^*}\big).
\]
Hence it is enough to control the latent conditional mutual information
\[
I\!\big(S^\parallel;S^\perp\mid X_{t^*}\big).
\]

the quantity
\[
I\!\big(S^\parallel;S^\perp\mid X_{t^*}\big)
\]
measures the residual dependence between the latent tangent and latent normal components
after observing \(S\) through a Gaussian channel of variance \(t^*+\varepsilon^2\). If
these two latent components are nearly conditionally independent given \(X_{t^*}\), then
the surrogate guided reverse dynamics are close to the true conditional reverse dynamics
in path-space KL.

In particular, if
\[
S^\parallel \perp S^\perp \mid X_{t^*},
\]
then
\[
I\!\big(S^\parallel;S^\perp\mid X_{t^*}\big)=0,
\]
and therefore also
\[
I\!\big(Z^\parallel;Z^\perp\mid X_{t^*}\big)=0.
\]
The upper bound in Theorem~\ref{thm:avg-path-kl} then implies that the average pathwise
KL discrepancy is zero.

More generally, the elementary bounds
\[
I\!\big(S^\parallel;S^\perp\mid X_{t^*}\big)
\le H\!\big(S^\perp\mid X_{t^*}\big),
\qquad
I\!\big(S^\parallel;S^\perp\mid X_{t^*}\big)
\le H\!\big(S^\parallel\mid X_{t^*}\big)
\]
show that the pathwise error is small whenever the safe-time observation \(X_{t^*}\)
almost determines the latent normal component \(S^\perp\) (equivalently, the latent affine
level), or almost determines the latent tangent component \(S^\parallel\). In particular,
if \(X_{t^*}\) already localizes the correct latent level well, then there is little room
for tangent/normal ambiguity, and the surrogate reverse dynamics are accurate.

The lower bound in Theorem~\ref{thm:avg-path-kl} remains informative as well: it shows that
the pathwise KL is governed by the same conditional-dependence mechanism, up to correction
terms involving coarser observations. Finally, the result remains dimension-free: the error
is controlled not explicitly by the ambient dimension \(d\) or the number of measurements
\(m\), but by the residual latent dependence between tangent and normal directions at the
safe time.
\end{remark}
\medskip

We now formalize the assumptions needed to control the initialization error at time \(t^*\)
and to express the pathwise term in latent information-theoretic form.


\begin{assumption}
\label{ass:quad_log_t0_kernel}
Define
\[
\Cc^\perp := P_\perp \Cc .
\]
For each \(t\ge t_0\) and \(c\in\Cc^\perp\), let
\[
r_t^c := \Law(X_t^\parallel \mid S^\perp  = c).
\]
We assume that for every \(t\ge t_0\) there exists a finite constant \(L_t<\infty\) such that
for all \(c_1,c_2\in\Cc^\perp\),
\begin{equation}
\label{eq:quad_log_t0_kernel}
\KL\!\left(r_t^{c_1}\,\middle\|\,r_t^{c_2}\right)
\le
L_t\,\|c_1-c_2\|_2^2 .
\end{equation}
\end{assumption}

\begin{remark}
Assumption~\ref{ass:quad_log_t0_kernel} says that, at each diffusion time \(t\ge t_0\),
the tangent conditional law \(\Law(X_t^\parallel\mid S^\perp=c)\) varies in a quantitatively
controlled way with the latent normal code \(c\in\Cc^\perp\), with sensitivity measured in
KL divergence.

Moreover, the map \(t\mapsto L_t\) is non-increasing. Indeed, if \(t\ge s\ge t_0\), then
\[
X_t^\parallel = X_s^\parallel + P_\parallel(W_t-W_s),
\]
where \(W_t-W_s\sim\Nc(0,(t-s)I_d)\) is independent of \((X_s,S)\). Hence, conditional on
\(S^\perp=c\), the law of \(X_t^\parallel\) is obtained from that of \(X_s^\parallel\) by the
same Gaussian Markov kernel \(K_{s,t}\), independent of \(c\):
\[
r_t^c = r_s^c K_{s,t}.
\]
Therefore, by the data-processing inequality for KL divergence,
\[
\KL(r_t^{c_1}\|r_t^{c_2})
=
\KL(r_s^{c_1}K_{s,t}\|r_s^{c_2}K_{s,t})
\le
\KL(r_s^{c_1}\|r_s^{c_2}).
\]
Dividing by \(\|c_1-c_2\|_2^2\) yields \(L_t\le L_s\).
\end{remark}

\medskip

We now state our main quantitative guarantee for the \emph{terminal tangent marginal}.
Recall that the ideal conditional reverse-time dynamics
\(\{Y_\tau^{*,b}\}_{\tau\in[\tau^*,\,T-t_0]}\), initialized from the true conditional marginal
at time \(\tau^*\), and the practical surrogate procedure
\(\{\hat Y_\tau^{\,b}\}_{\tau\in[\tau^*,\,T-t_0]}\), obtained by the two-stage initialization
at time \(t^*\) followed by the surrogate guided reverse dynamics, induce terminal tangent laws
\[
\mu_{T-t_0}^{*,b}
:= \Law\!\big(P_\parallel Y_{T-t_0}^{*,b}\big),
\qquad
\hat\mu_{T-t_0}^{\,b}
:= \Law\!\big(P_\parallel \hat Y_{T-t_0}^{\,b}\big).
\]
Our goal is to bound the averaged terminal discrepancy
\[
\E_B\!\left[\KL\!\big(\mu_{T-t_0}^{*,B}\,\big\|\,\hat\mu_{T-t_0}^{\,B}\big)\right].
\]

\begin{theorem}
\label{thm:total_KL_shannon_small}
Let Assumptions~\ref{assumption0}, \ref{ass:latent_gaussian_mixture}, and
\ref{ass:quad_log_t0_kernel} be in force. Let
\[
H := H(S^\perp),
\]
and assume \(H<\infty\). Fix a safe noise level \(t^*\in(t_0,T)\)
(equivalently, \(\tau^*=T-t^*\)). Then
\begin{equation}
\label{eq:total_KL_shannon_small}
\E_B\!\left[\KL\!\big(\mu_{T-t_0}^{*,B}\,\middle\|\,\hat\mu_{T-t_0}^{\,B}\big)\right]
\le
4L_{t^*}(t^*+\varepsilon^2)\,H
+
I\!\big(S^\parallel;\,S^\perp \mid X_{t^*}\big).
\end{equation}
\end{theorem}

\begin{remark}
The bound \eqref{eq:total_KL_shannon_small} separates two distinct sources of error.

The first term,
\[
4L_{t^*}(t^*+\varepsilon^2)\,H,
\]
is the \emph{initialization error} at the safe time \(t^*\). It reflects the discrepancy
between the true conditional marginal at time \(t^*\) and the two-stage approximation used
to initialize the reverse dynamics. The factor \(L_{t^*}\) measures how sensitive the tangent
conditional law is to changes in the latent normal code, while \(H=H(S^\perp)\) measures the
complexity of the latent normal-code prior.

The second term,
\[
I\!\big(S^\parallel;\,S^\perp \mid X_{t^*}\big),
\]
is the \emph{pathwise error}. By Theorem~\ref{thm:avg-path-kl},
\[
\E_B\!\big[\KL(\PP^{Y^{*,B}}\|\PP^{\hat Y^{\,B}})\big]
\le
I\!\big(Z^\parallel;Z^\perp\mid X_{t^*}\big),
\]
and Proposition~\ref{prop:latent_cmi_dpi} gives
\[
I\!\big(Z^\parallel;Z^\perp\mid X_{t^*}\big)
\le
I\!\big(S^\parallel;S^\perp\mid X_{t^*}\big).
\]
Thus the pathwise discrepancy is controlled by the residual dependence between the latent
tangent and latent normal components after observation through the effective Gaussian channel.
\end{remark}


Theorem~\ref{thm:total_KL_shannon_small} applies without any geometric separation assumption
on the set of admissible latent normal codes \(\Cc^\perp\). In that general case, the noisy
normal observation may remain ambiguous among several nearby latent codes. We now show that,
if the admissible codes are uniformly separated, then such confusions become rare and both
the initialization and pathwise contributions become exponentially small.

\begin{assumption}
\label{ass:delta_sep_S}
There exists \(\delta>0\) such that for all distinct \(c,\tilde c\in\Cc^\perp\),
\[
\|c-\tilde c\|_2 \ge \delta.
\]
\end{assumption}

Assumption~\ref{ass:delta_sep_S} enforces a minimum spacing between admissible latent normal
codes. Since
\[
X_{t^*}^\perp = S^\perp + \sqrt{t^*+\varepsilon^2}\,G,
\qquad G\sim\Nc(0,I_d),
\]
confusing the true latent code \(c\) with a different code \(\tilde c\) requires a Gaussian
fluctuation of order at least \(\delta\), which occurs with probability
\(\exp(-\Omega(\delta^2/(t^*+\varepsilon^2)))\). This separation upgrades the Shannon-scale
control above to exponentially small error bounds.

\begin{theorem}
\label{thm:total_error_sep_renyi_single}
Let Assumptions~\ref{assumption0}, \ref{ass:latent_gaussian_mixture},
\ref{ass:quad_log_t0_kernel}, and \ref{ass:delta_sep_S} be in force. Let
\[
H_{1/2} := H_{1/2}(S^\perp)
=
2\log \sum_{c\in\Cc^\perp}\sqrt{p_{S^\perp}(c)},
\qquad
\sigma_*^2 := t^*+\varepsilon^2,
\]
and fix a safe noise level \(t^*\in(t_0,T)\). Then
\begin{align}
\label{eq:total_error_sep_renyi_single}
\E_B\!\Big[
\KL\!\big(p_{t^*}^{*,B}\,\Vert\,\hat p_{t^*}^{\,B}\big)
+
\KL\!\big(\PP^{Y^{*,B}}\Vert \PP^{\hat Y^{\,B}}\big)
\Big]
\le\;&
L_{t^*}\Big(\frac{\delta^2}{2}+4\sigma_*^2\Big)
\exp\!\Big(H_{1/2}-\frac{\delta^2}{8\sigma_*^2}\Big)
\nonumber\\
&\quad
+
2\exp\!\Big(H_{1/2}-\frac{\delta^2}{8\sigma_*^2}\Big).
\end{align}
Consequently, by data processing,
\begin{equation}
\label{eq:terminal_sep_renyi_single}
\E_B\!\Big[
\KL\!\big(\mu_{T-t_0}^{*,B}\,\Vert\,\hat\mu_{T-t_0}^{\,B}\big)
\Big]
\le
\Big[
L_{t^*}\Big(\frac{\delta^2}{2}+4\sigma_*^2\Big)+2
\Big]
\exp\!\Big(H_{1/2}-\frac{\delta^2}{8\sigma_*^2}\Big).
\end{equation}
\end{theorem}

\begin{remark}
Theorem~\ref{thm:total_error_sep_renyi_single} is an exponential strengthening of
Theorem~\ref{thm:total_KL_shannon_small}. Under \(\delta\)-separation, the noisy normal
observation
\[
X_{t^*}^\perp = S^\perp + \sqrt{t^*+\varepsilon^2}\,G
\]
can confuse two distinct admissible latent codes only if the Gaussian perturbation produces
a normal displacement of size at least \(\delta\). This yields the exponential factor
\[
\exp\!\Big(H_{1/2}-\frac{\delta^2}{8\sigma_*^2}\Big)
\]
appearing in both \eqref{eq:total_error_sep_renyi_single} and
\eqref{eq:terminal_sep_renyi_single}.

In \eqref{eq:total_error_sep_renyi_single}, the first term
\[
L_{t^*}\Big(\frac{\delta^2}{2}+4\sigma_*^2\Big)
\exp\!\Big(H_{1/2}-\frac{\delta^2}{8\sigma_*^2}\Big)
\]
controls the \emph{initialization error}: separation forces the posterior over latent normal
codes, given \(X_{t^*}^\perp\), to concentrate near the true code, and
Assumption~\ref{ass:quad_log_t0_kernel} converts this concentration into a KL bound through
the sensitivity constant \(L_{t^*}\).

The second term
\[
2\exp\!\Big(H_{1/2}-\frac{\delta^2}{8\sigma_*^2}\Big)
\]
controls the \emph{pathwise error}: as the separation-to-noise ratio
\(\delta^2/\sigma_*^2\) increases, the normal observation essentially identifies the latent
normal code, so the surrogate guided reverse dynamics become close to the ideal conditional
reverse dynamics.

Consequently, once \(\delta^2/\sigma_*^2\) is sufficiently large compared with the effective
complexity scale \(H_{1/2}(S^\perp)\), both contributions are simultaneously small and the
expected KL error decays exponentially in the separation-to-noise ratio.
\end{remark}
\acks{
Funding in direct support of this work: none.
Competing interests and additional revenues related to this work: the authors declare no competing interests.
}
\newpage

\appendix

\section{ Proof of Theorem~\ref{thm:avg-path-kl} }

\noindent

We compare the true conditional reverse-time dynamics and the surrogate guided dynamics at the level of path measures. Since the two SDEs have the same diffusion coefficient and differ only in the tangent drift, the first task is to justify a Girsanov formula for their relative entropy. For this, in Lemma~\ref{lem:posterior-mean-linear-growth} we first prove that the posterior mean \(x\mapsto \E[Z\mid X_t=x]\) has at most linear growth; via Tweedie’s formula, this implies the required linear-growth control on the two drifts. Then Lemma~\ref{thm:girsanov-linear-growth} converts the pathwise KL divergence into an integral of the squared drift gap along the true conditional path. 

The next step is to rewrite this drift gap in statistical terms. Using Tweedie’s identity and projecting onto the tangent space, the drift difference becomes the difference between two posterior means of the tangent component \(U=P_\parallel Z\): one conditioned on the noisy observation alone, and one conditioned on the noisy observation together with the normal component \(B=P_\perp Z\). Averaging over the random level \(B\) and applying the MMSE-gap identity turns the pathwise KL bound into an integral of conditional MMSE differences. The conditional I--MMSE lemma is then used to identify this integral with a difference of conditional mutual informations, yielding the upper bound in terms of \(I(U;B\mid X_{t^*})\). 

For the lower bound, the same MMSE representation is kept on the finite interval corresponding to \(t\in[t_0,t^*]\). One then isolates the error term involving the MMSE of the normal component \(B\), and controls it by projecting the observation onto the normal subspace. The key observation is that, conditional on \(U\), the parallel observation carries no information about \(B\), so the relevant MMSE gap can be reduced to a Gaussian channel only in the normal directions. Applying the conditional I--MMSE identity once more to this reduced channel yields the correction term involving \(I(U;B\mid X_{t^*}^\perp)\), and this gives the stated lower bound. 
\begin{lemma}\label{lem:posterior-mean-linear-growth}
Let assumption~\ref{assumption0} be in force. Fix $t\ge 0$ and set
\[
m_t(x):=\mathbb E[Z\mid X_t=x],\qquad x\in\mathbb R^d.
\]
Then there exists a constant $C_t<\infty$ such that
\[
|m_t(x)|\le C_t(1+|x|),\qquad x\in\mathbb R^d.
\]
In particular, the posterior mean $x\mapsto \mathbb E[Y\mid X_t=x]$ has at most linear growth.
\end{lemma}

\begin{proof}
Let $\mu:=\mathrm{Law}(Z)$, and let
\[
\phi_t(u):=(2\pi t)^{-d/2}\exp\!\left(-\frac{|u|^2}{2t}\right),\qquad u\in\mathbb R^d,
\]
be the Gaussian kernel with covariance matrix $tI_d$. The law of $X_t$ admits density
\[
p_t(x)=\int_{\mathbb R^d}\phi_t(x-z)\,\mu(dz),
\]
and the conditional mean is given by
\[
m_t(x)
=
\frac{\int_{\mathbb R^d} z\,\phi_t(x-z)\,\mu(dz)}
{\int_{\mathbb R^d}\phi_t(x-z)\,\mu(dz)}.
\]
Set
\[
N(x):=\int_{\mathbb R^d} z\,\phi_t(x-z)\,\mu(dz),
\qquad
D(x):=\int_{\mathbb R^d}\phi_t(x-z)\,\mu(dz)=p_t(x).
\]
Then
\[
m_t(x)=\frac{N(x)}{D(x)}.
\]

We shall prove that
\[
|N(x)|\le C_t(1+|x|)D(x),\qquad x\in\mathbb R^d.
\]

Choose $R>0$ such that
\[
a:=\mu(B(0,R))>0.
\]
This is possible since $\mu$ is a probability measure. Fix $x\in\mathbb R^d$. We split the numerator into a near part and a far part:
\[
N(x)=N_1(x)+N_2(x),
\]
where
\[
N_1(x):=\int_{\{|z|\le 4(|x|+R)\}} z\,\phi_t(x-z)\,\mu(dz),
\]
and
\[
N_2(x):=\int_{\{|z|>4(|x|+R)\}} z\,\phi_t(x-z)\,\mu(dz).
\]

\medskip

On the set $\{|z|\le 4(|x|+R)\}$ one has $|z|\le 4(|x|+R)$, and therefore
\[
|N_1(x)|
\le \int_{\{|z|\le 4(|x|+R)\}} |z|\,\phi_t(x-z)\,\mu(dz)
\le 4(|x|+R)D(x).
\]

\medskip

Let $z\in\mathbb R^d$ satisfy $|z|>4(|x|+R)$, and let $u\in B(0,R)$. Then
\[
|x-u|\le |x|+|u|\le |x|+R<\frac{|z|}{4},
\]
while
\[
|x-z|\ge |z|-|x|>|z|-\frac{|z|}{4}=\frac{3}{4}|z|.
\]
Hence
\[
|x-z|^2-|x-u|^2
\ge \frac{9}{16}|z|^2-\frac{1}{16}|z|^2
=\frac12 |z|^2.
\]
Consequently,
\[
\frac{\phi_t(x-z)}{\phi_t(x-u)}
=
\exp\!\left(-\frac{|x-z|^2-|x-u|^2}{2t}\right)
\le \exp\!\left(-\frac{|z|^2}{4t}\right).
\]
Thus,
\[
\phi_t(x-z)\le e^{-|z|^2/(4t)}\,\phi_t(x-u),\qquad u\in B(0,R).
\]
Integrating this inequality with respect to $\mu(du)$ over $B(0,R)$ gives
\[
a\,\phi_t(x-z)
\le e^{-|z|^2/(4t)}\int_{B(0,R)}\phi_t(x-u)\,\mu(du)
\le e^{-|z|^2/(4t)}D(x),
\]
and therefore
\[
\phi_t(x-z)\le a^{-1}e^{-|z|^2/(4t)}D(x).
\]
Using this bound, we obtain
\begin{align*}
|N_2(x)|
&\le \int_{\{|z|>4(|x|+R)\}} |z|\,\phi_t(x-z)\,\mu(dz)\\
&\le a^{-1}D(x)\int_{\mathbb R^d}|z|e^{-|z|^2/(4t)}\,\mu(dz).
\end{align*}
Since the function $r\mapsto re^{-r^2/(4t)}$ is bounded on $[0,\infty)$, the quantity
\[
C_{t,1}:=a^{-1}\int_{\mathbb R^d}|z|e^{-|z|^2/(4t)}\,\mu(dz)
\]
is finite. Hence
\[
|N_2(x)|\le C_{t,1}D(x).
\]

\medskip
Combining the bounds for $N_1(x)$ and $N_2(x)$, we get
\[
|N(x)|
\le \bigl(4(|x|+R)+C_{t,1}\bigr)D(x).
\]
Dividing by $D(x)>0$, we conclude that
\[
|m_t(x)|
=
\left|\frac{N(x)}{D(x)}\right|
\le 4|x|+4R+C_{t,1}.
\]
Therefore there exists a finite constant $C_t$ such that
\[
|m_t(x)|\le C_t(1+|x|),\qquad x\in\mathbb R^d.
\]
This completes the proof.
\end{proof}
\begin{lemma}\label{thm:girsanov-linear-growth}
Let $T>0$, let $\Omega=C([0,T];\mathbb R^d)$ be endowed with the canonical filtration
$(\mathcal F_t)_{0\le t\le T}$, and let $X_t(\omega)=\omega(t)$ be the coordinate process.

Assume that $b,\beta:[0,T]\times \mathbb R^d\to\mathbb R^d$ are Borel measurable and satisfy
\[
|b(t,x)|+|\beta(t,x)|\le L(1+|x|), \qquad (t,x)\in[0,T]\times\mathbb R^d,
\]
for some constant $L>0$. Let $\sigma\in\mathbb R^{d\times d}$ be a constant invertible matrix, and let $\nu$ be a probability measure on $\mathbb R^d$ such that
\[
\int_{\mathbb R^d}|x|^2\,\nu(dx)<\infty.
\]

Suppose that $\mathbb{P}^\beta$ is a weak solution law of
\[
dX_t=\beta(t,X_t)\,dt+\sigma\,dW_t,\qquad X_0\sim \nu.
\]
Equivalently, under $\mathbb{P}^\beta$,
\[
W_t^\beta:=\sigma^{-1}\Bigl(X_t-X_0-\int_0^t \beta(s,X_s)\,ds\Bigr)
\]
is a $d$-dimensional Brownian motion.

Define
\[
\theta(t,x):=\sigma^{-1}(b-\beta)(t,x),
\]
and
\[
Z_t
:=
\exp\!\left(
\int_0^t \theta(s,X_s)\cdot dW_s^\beta
-\frac12\int_0^t |\theta(s,X_s)|^2\,ds
\right),\qquad 0\le t\le T.
\]

Then the following hold:
\begin{enumerate}
\item $Z=(Z_t)_{0\le t\le T}$ is a true $\mathbb{P}^\beta$-martingale;
\item the probability measure $\mathbb{P}^b$ on $(\Omega,\mathcal F_T)$ defined by
\[
\frac{d\mathbb{P}^b}{d\mathbb{P}^\beta}=Z_T
\]
is a weak solution law of
\[
dX_t=b(t,X_t)\,dt+\sigma\,dW_t,\qquad X_0\sim \nu.
\]
\end{enumerate}
In particular,
\[
\mathbb{P}^b\ll \mathbb{P}^\beta \quad\text{on }\mathcal F_T,
\]
with Radon--Nikodym derivative $Z_T$.

If, in addition, the martingale problem for $(b,\sigma,\nu)$ is well posed, then $\mathbb{P}^b$ is the unique weak solution law of the $b$-equation. If both martingale problems $(\beta,\sigma,\nu)$ and $(b,\sigma,\nu)$ are well posed, then $\mathbb{P}^\beta$ and $\mathbb{P}^b$ are equivalent on $\mathcal F_T$.
\end{lemma}

\begin{proof}
We divide the argument into several steps.

\medskip
For each $n\in\mathbb N$, define the stopping time
\[
\tau_n:=\inf\{t\in[0,T]: |X_t|\ge n\}\wedge T,
\]
and set
\[
Z_t^{(n)}:=Z_{t\wedge \tau_n}.
\]
Since the process
\[
\theta_n(t):=\theta(t,X_t)\mathbf 1_{\{t\le \tau_n\}}
\]
is bounded, the classical bounded-integrand version of Girsanov's theorem implies that $Z^{(n)}$ is a true $\mathbb{P}^\beta$-martingale. Define a probability measure $\mathbb{Q}_n$ on $\mathcal F_T$ by
\[
\frac{d\mathbb{Q}_n}{d\mathbb{P}^\beta}=Z_T^{(n)}.
\]

Under $\mathbb{Q}_n$, the process
\[
W_t^{(n)}:=W_t^\beta-\int_0^{t\wedge\tau_n}\theta(s,X_s)\,ds
\]
is a $d$-dimensional Brownian motion. Consequently,
\begin{align*}
X_{t\wedge\tau_n}
&=X_0+\int_0^{t\wedge\tau_n}\beta(s,X_s)\,ds+\sigma W_{t\wedge\tau_n}^\beta\\
&=X_0+\int_0^{t\wedge\tau_n}\beta(s,X_s)\,ds
+\sigma\int_0^{t\wedge\tau_n}\theta(s,X_s)\,ds
+\sigma W_{t\wedge\tau_n}^{(n)}\\
&=X_0+\int_0^{t\wedge\tau_n}b(s,X_s)\,ds+\sigma W_{t\wedge\tau_n}^{(n)}.
\end{align*}
Thus, under $\mathbb{Q}_n$, the stopped coordinate process solves the $b$-equation up to $\tau_n$. Define
\[
f_n(t):=E^{Q_n}\Big[\sup_{0\le u\le t}|X_{u\wedge\tau_n}|^2\Big],\qquad 0\le t\le T.
\]
Write
\[
X_{t\wedge\tau_n}=X_0+A_t^{(n)}+M_t^{(n)},
\]
where
\[
A_t^{(n)}:=\int_0^{t\wedge\tau_n}b(s,X_s)\,ds,
\qquad
M_t^{(n)}:=\sigma W_{t\wedge\tau_n}^{(n)}.
\]
Using $(a+b+c)^2\le 3(a^2+b^2+c^2)$, we obtain
\[
f_n(t)\le 3E^{\mathbb{Q}_n}|X_0|^2
+3E^{\mathbb{Q}_n}\Big[\sup_{u\le t}|A_u^{(n)}|^2\Big]
+3E^{\mathbb{Q}_n}\Big[\sup_{u\le t}|M_u^{(n)}|^2\Big].
\]

Since $Z_0^{(n)}=1$, the law of $X_0$ under $\mathbb{Q}_n$ is the same as under $\mathbb{P}^\beta$, namely $\nu$. Hence
\[
E^{\mathbb{Q}_n}|X_0|^2=\int_{\mathbb R^d}|x|^2\,\nu(dx)<\infty.
\]

For the drift term, the linear-growth assumption yields
\[
|b(t,x)|\le L(1+|x|),
\]
so
\[
\sup_{u\le t}|A_u^{(n)}|
\le \int_0^t \mathbf 1_{\{s\le \tau_n\}} |b(s,X_s)|\,ds
\le L\int_0^t \bigl(1+|X_{s\wedge\tau_n}|\bigr)\,ds.
\]
By Cauchy--Schwarz,
\[
\Big(\int_0^t \bigl(1+|X_{s\wedge\tau_n}|\bigr)\,ds\Big)^2
\le t\int_0^t \bigl(1+|X_{s\wedge\tau_n}|\bigr)^2\,ds
\le 2t\int_0^t \bigl(1+|X_{s\wedge\tau_n}|^2\bigr)\,ds.
\]
Therefore,
\[
E^{\mathbb{Q}_n}\Big[\sup_{u\le t}|A_u^{(n)}|^2\Big]
\le 2L^2 t\int_0^t \Bigl(1+E^{\mathbb{Q}_n}|X_{s\wedge\tau_n}|^2\Bigr)\,ds
\le 2L^2 t\int_0^t \bigl(1+f_n(s)\bigr)\,ds.
\]

For the martingale term, $M^{(n)}$ is a continuous $\mathbb{Q}_n$-martingale with quadratic variation
\[
\langle M^{(n)}\rangle_t
=
\int_0^{t\wedge\tau_n}\sigma\sigma^\top\,ds.
\]
Hence, by the Burkholder--Davis--Gundy inequality\citep{karatzas2014brownian},
\[
E^{\mathbb{Q}_n}\Big[\sup_{u\le t}|M_u^{(n)}|^2\Big]
\le C_{\mathrm{BDG}}\,E^{\mathbb{Q}_n}\big[\mathrm{tr}\langle M^{(n)}\rangle_t\big]
\le C_{\mathrm{BDG}}\|\sigma\|_{\mathrm{HS}}^2\,t.
\]
Combining the above bounds, we find constants $C_0,C_1>0$, independent of $n$, such that
\[
f_n(t)\le C_0+C_1\int_0^t \bigl(1+f_n(s)\bigr)\,ds,\qquad 0\le t\le T.
\]
By Gronwall's lemma,
\[
\sup_{n\ge 1}\sup_{0\le t\le T} f_n(t)\le C_T
\]
for some constant $C_T<\infty$ independent of $n$. In particular,
\[
\sup_{n\ge 1}E^{\mathbb{Q}_n}\int_0^T |X_{s\wedge\tau_n}|^2\,ds
\le T C_T.
\]
Under $\mathbb{Q}_n$,
\[
dW_t^\beta=dW_t^{(n)}+\mathbf 1_{\{t\le \tau_n\}}\theta(t,X_t)\,dt.
\]
Substituting into the definition of $Z_T^{(n)}$ yields
\[
\log Z_T^{(n)}
=
\int_0^{T\wedge\tau_n}\theta(s,X_s)\cdot dW_s^{(n)}
+\frac12\int_0^{T\wedge\tau_n}|\theta(s,X_s)|^2\,ds.
\]
Taking expectations under $\mathbb{Q}_n$, the stochastic integral has mean zero, and therefore
\[
E^{\mathbb{Q}_n}[\log Z_T^{(n)}]
=
\frac12E^{\mathbb{Q}_n}\int_0^{T\wedge\tau_n}|\theta(s,X_s)|^2\,ds.
\]
Since $\theta(t,x)=\sigma^{-1}(b-\beta)(t,x)$ and both $b$ and $\beta$ have linear growth, there exists a constant $C>0$ such that
\[
|\theta(t,x)|^2\le C(1+|x|^2).
\]
Hence
\[
E^{\mathbb{Q}_n}[\log Z_T^{(n)}]
\le C\left(T+E^{\mathbb{Q}_n}\int_0^T |X_{s\wedge\tau_n}|^2\,ds\right)
\le C_T'.
\]
Moreover, by definition of $\mathbb{Q}_n$,
\[
E^{\mathbb{P}^\beta}\big[Z_T^{(n)}\log Z_T^{(n)}\big]
=
E^{\mathbb{Q}_n}[\log Z_T^{(n)}].
\]
Thus,
\[
\sup_{n\ge 1}E^{\mathbb{P}^\beta}\big[Z_T^{(n)}\log Z_T^{(n)}\big]<\infty.
\]
Since the function $x\mapsto x\log x$ is increasing convex function  , de la Vallée-Poussin's criterion implies that the family $\{Z_T^{(n)}\}_{n\ge 1}$ is uniformly integrable \citep{durrett2019probability}.

Now $\tau_n\uparrow T$ $\mathbb{P}^\beta$-almost surely, and $Z$ is continuous, hence
\[
Z_T^{(n)}\to Z_T
\qquad \mathbb{P}^\beta\text{-a.s.}
\]
Uniform integrability therefore implies
\[
E^{\mathbb{P}^\beta}[Z_T]
=
\lim_{n\to\infty}E^{\mathbb{P}^\beta}[Z_T^{(n)}]
=
1.
\]
It follows that $Z$ is a true $\mathbb{P}^\beta$-martingale.Now define a probability measure $\mathbb{P}^b$ on $\mathcal F_T$ by
\[
\frac{d\mathbb{P}^b}{d\mathbb{P}^\beta}=Z_T.
\]
Since $Z$ is a true martingale, the classical Girsanov theorem applies and yields that
\[
W_t^b:=W_t^\beta-\int_0^t \theta(s,X_s)\,ds
\]
is a Brownian motion under $\mathbb{P}^b$. Therefore,
\begin{align*}
X_t
&=X_0+\int_0^t \beta(s,X_s)\,ds+\sigma W_t^\beta\\
&=X_0+\int_0^t \beta(s,X_s)\,ds
+\sigma\int_0^t \theta(s,X_s)\,ds
+\sigma W_t^b\\
&=X_0+\int_0^t b(s,X_s)\,ds+\sigma W_t^b.
\end{align*}
Thus $X$ solves
\[
dX_t=b(t,X_t)\,dt+\sigma\,dW_t^b
\]
under $\mathbb{P}^b$.

It remains to identify the initial law. For every Borel set $A\subset\mathbb R^d$,
\[
\mathbb{P}^b(X_0\in A)
=
E^{\mathbb{P}^\beta}\big[\mathbf 1_{\{X_0\in A\}}Z_T\big].
\]
Since $\mathbf 1_{\{X_0\in A\}}\in\mathcal F_0$ and $Z$ is a martingale with $Z_0=1$,
\[
E^{\mathbb{P}^\beta}\big[\mathbf 1_{\{X_0\in A\}}Z_T\big]
=
E^{\mathbb{P}^\beta}\Big[\mathbf 1_{\{X_0\in A\}}E^{\mathbb{P}^\beta}[Z_T\mid \mathcal F_0]\Big]
=
E^{\mathbb{P}^\beta}\big[\mathbf 1_{\{X_0\in A\}}\big]
=
\nu(A).
\]
Hence $X_0\sim \nu$ under $\mathbb{P}^b$ as well. This proves that $\mathbb{P}^b$ is a weak solution law of
\[
dX_t=b(t,X_t)\,dt+\sigma\,dW_t,\qquad X_0\sim \nu.
\]

\medskip

If the martingale problem for $(b,\sigma,\nu)$ is well posed, then the probability measure $\mathbb{P}^b$ constructed above coincides with the unique weak solution law of the $b$-equation. If both martingale problems $(\beta,\sigma,\nu)$ and $(b,\sigma,\nu)$ are well posed, then the same argument with $b$ and $\beta$ interchanged yields
\[
\mathbb{P}^\beta\ll \mathbb{P}^b,
\]
and therefore
\[
\mathbb{P}^\beta\sim \mathbb{P}^b.
\]

This completes the proof.
\end{proof}

\begin{lemma}\label{lem:mmse_gap}
Let $U$ be square-integrable and let $\mathcal G \subseteq \mathcal H$ be
$\sigma$-fields. Then
\[
\E\bigl[\| \E[U\mid \mathcal H]-\E[U\mid \mathcal G]\|_2^2\bigr]
=\mmse(U\mid \mathcal G)-\mmse(U\mid \mathcal H),
\]
where $\mmse(U\mid \mathcal G):=\E\|U-\E[U\mid\mathcal G]\|_2^2$.

\end{lemma}

\begin{proof}
The identity is the Pythagorean theorem for orthogonal projections in
$L^2$ (conditional expectation is the orthogonal projection onto the subspace
of $\mathcal G$-measurable functions).

\end{proof}
\begin{lemma}\label{lem:conditional_immse}
Let $X\in\R^d$ be a random vector with $\E\|X\|_2^2<\infty$, let $S$ be an arbitrary
random element, and let $N\sim\Nc(0,I_d)$ be independent of $(X,S)$. For $\gamma>0$
define the Gaussian observation channel
\[
Y_\gamma \;:=\; \sqrt{\gamma}\,X + N .
\]
Then $I(X;Y_\gamma\mid S)$ is differentiable in $\gamma$ and
\[
\frac{\dd}{\dd\gamma}\,I(X;Y_\gamma\mid S)
\;=\;
\frac12\,\mmse(X\mid Y_\gamma,S),
\]
where
\[
\mmse(X\mid Y_\gamma,S)
\;:=\;
\E\!\left[\big\|X-\E[X\mid Y_\gamma,S]\big\|_2^2\right].
\]
\end{lemma}

\begin{proof}
Fix $\gamma>0$. By disintegration,
\begin{equation}\label{eq:cond_mi_disintegrate}
I(X;Y_\gamma\mid S)=\int I(X;Y_\gamma\mid S=s)\,P_S(\dd s).
\end{equation}
For each $s$, conditional on $S=s$ the channel remains AWGN:
$Y_\gamma=\sqrt{\gamma}\,X+N$ with $N\perp\!\!\!\perp X$ under $\Law(\cdot\mid S=s)$.
Since $\E[\|X\|_2^2\mid S=s]<\infty$ for $P_S$-a.e.\ $s$, the (vector) I--MMSE
identity of \citep{guo2005mutual} applied to the conditional input law $X\mid S=s$
yields
\begin{equation}\label{eq:immse_cond_s}
\frac{\dd}{\dd\gamma}I(X;Y_\gamma\mid S=s)
=
\frac12\,\mmse(X\mid Y_\gamma,S=s).
\end{equation}
Moreover, for every $s$,
\[
0\le \mmse(X\mid Y_\gamma,S=s)\le \E[\|X\|_2^2\mid S=s],
\]
because the MMSE is the minimum mean-squared error and is upper bounded by the MSE
of the zero estimator. Since $\E\|X\|_2^2=\int \E[\|X\|_2^2\mid S=s]\,P_S(\dd s)<\infty$,
dominated convergence (Leibniz rule) allows differentiating under the integral in
\eqref{eq:cond_mi_disintegrate}, giving
\[
\frac{\dd}{\dd\gamma}I(X;Y_\gamma\mid S)
=
\frac12\int \mmse(X\mid Y_\gamma,S=s)\,P_S(\dd s).
\]
Finally, by the law of total expectation and the definition of conditional MMSE,
\[
\int \mmse(X\mid Y_\gamma,S=s)\,P_S(\dd s)
=
\E\!\left[\big\|X-\E[X\mid Y_\gamma,S]\big\|_2^2\right]
=
\mmse(X\mid Y_\gamma,S),
\]
which proves the claim.
\end{proof}
\begin{proof}[Proof of Theorem~\ref{thm:avg-path-kl}]
Fix \(b\). Let \(Y^{*,b}\) and \(\hat Y^{\,b}\) solve the two reverse-time SDEs on
\([\tau^*,T-t_0]\) in the theorem, started from the same law at time \(\tau^*\) and with the
same unit diffusion coefficient. Denote their drifts by \(f_\tau^{*,b}\) and
\(\hat f_\tau^{\,b}\), respectively.

We first explain why these drifts have at most linear growth and why
Lemma~\ref{thm:girsanov-linear-growth} applies. By Tweedie’s formula, for \(t=T-\tau\),
\[
\E[Z\mid X_t=x]=x+t\,s_t(x),
\qquad
\E[Z\mid X_t=x,B=b]=x+t\,s_t^{*,b}(x).
\]
Hence
\[
f_\tau^{*,b}(x)
=
P_\parallel s_t^{*,b}(x)+\frac{1}{t}P_\perp(b-x)
=
\frac{1}{t}\bigl(\E[Z\mid X_t=x,B=b]-x\bigr),
\]
because under the conditioning \(B=b\) we have
\(P_\perp \E[Z\mid X_t=x,B=b]=b\). Likewise,
\[
\hat f_\tau^{\,b}(x)
=
P_\parallel s_t(x)+\frac{1}{t}P_\perp(b-x)
=
\frac{1}{t}\bigl(P_\parallel \E[Z\mid X_t=x]-P_\parallel x + P_\perp(b-x)\bigr).
\]
By Lemma~\ref{lem:posterior-mean-linear-growth}, the maps
\(x\mapsto \E[Z\mid X_t=x]\) and \(x\mapsto \E[Z\mid X_t=x,B=b]\) have at most linear growth.
Since \(t=T-\tau\in[t_0,t^*]\) on \([\tau^*,T-t_0]\), the factor \(1/t\) is uniformly bounded by
\(1/t_0\). Therefore both \(f_\tau^{*,b}\) and \(\hat f_\tau^{\,b}\) satisfy a linear-growth bound
of the form
\[
|f_\tau^{*,b}(x)|+|\hat f_\tau^{\,b}(x)|
\le C_{b,t_0}(1+|x|),
\qquad \tau\in[\tau^*,T-t_0].
\]

Now \(Y^{*,b}\) is already given as a weak solution law for the drift \(f^{*,b}\), namely the
ideal conditional reverse process. We therefore apply
Lemma~\ref{thm:girsanov-linear-growth} with
\[
\beta(\tau,x)=f_\tau^{*,b}(x),
\qquad
b(\tau,x)=\hat f_\tau^{\,b}(x),
\qquad
\sigma=I_d,
\]
and with initial law equal to the common law of \(Y^{*,b}_{\tau^*}\) and \(\hat Y^{\,b}_{\tau^*}\).
The lemma yields existence of the surrogate weak solution law and the relative-entropy identity
\begin{align}
\label{eq:pathKL_girsanov_clean}
\KL\!\big(\PP^{Y^{*,b}}\big\|\PP^{\hat Y^{\,b}}\big)
&=
\frac12\,\E^{Y^{*,b}}\!\left[\int_{\tau^*}^{T-t_0}
\big\|f_\tau^{*,b}(Y_\tau^{*,b})-\hat f_\tau^{\,b}(Y_\tau^{*,b})\big\|_2^2\,\dd \tau\right] \nonumber\\
&\le
\frac12\,\E^{Y^{*,b}}\!\left[\int_{\tau^*}^{T}
\big\|f_\tau^{*,b}(Y_\tau^{*,b})-\hat f_\tau^{\,b}(Y_\tau^{*,b})\big\|_2^2\,\dd \tau\right].
\end{align}
Since \(Y^{*,b}\) is the true reverse-time process of the conditional forward diffusion
\(\{X_t\}_{t\in[0,T]}\) under \(B=b\), it is defined up to terminal reverse time \(T\).
Hence we may extend the integral from \([\tau^*,T-t_0]\) to \([\tau^*,T)\).
By inspection, the two SDEs have the same normal drift, hence for all \(x\in\R^d\),
\begin{equation}\label{eq:drift_gap_parallel}
f_\tau^{*,b}(x)-\hat f_\tau^{\,b}(x)
=
P_\parallel\Big(s_t^{*,b}(x)-s_t(x)\Big).
\end{equation}

Now let \(Z\sim p_0\) denote the clean signal, and decompose it as
\[
U:=P_\parallel Z,
\qquad
B:=P_\perp Z.
\]

Tweedie's formula yields, for every \(x\in\R^d\),
\[
\E[Z\mid X_t=x]
=
x+t\,s_t(x),
\qquad
\E[Z\mid X_t=x,\,B=b]
=
x+t\,s_t^{*,b}(x).
\]
Projecting onto \(\ker(A)\) and subtracting, we obtain
\begin{equation}\label{eq:tweedie_parallel_gap_Slean}
P_\parallel\!\big(s_t^{*,b}-s_t\big)(x)
=
\frac1t\Big(\E[U\mid X_t=x,\,B=b]-\E[U\mid X_t=x]\Big).
\end{equation}
Setting \(t=T-\tau\) and combining \eqref{eq:drift_gap_parallel} with
\eqref{eq:tweedie_parallel_gap_Slean}, we get
\[
f_\tau^{*,b}(x)-\hat f_\tau^{\,b}(x)
=
\frac{1}{T-\tau}\Big(\E[U\mid X_{T-\tau}=x,\,B=b]-\E[U\mid X_{T-\tau}=x]\Big).
\]

Plugging this identity into \eqref{eq:pathKL_girsanov_clean}, averaging over the random level
\(B\), and applying Lemma~\ref{lem:mmse_gap}, yields
\begin{equation}\label{eq:KL_as_mmse_tau}
\E_B\!\Big[\KL\!\big(\PP^{Y^{*,B}}\big\|\PP^{\hat Y^{\,B}}\big)\Big]
\le
\frac12\int_{\tau^*}^{T}
\frac{1}{(T-\tau)^2}
\Big(
\mmse(U\mid X_{T-\tau})-\mmse(U\mid X_{T-\tau},B)
\Big)\,\dd\tau .
\end{equation}

Now define
\[
\gamma:=\frac1t=\frac{1}{T-\tau},
\qquad
\tilde X_\gamma:=\sqrt{\gamma}\,X_t
=
\sqrt{\gamma}\,Z+\Xi,
\qquad
\Xi\sim\Nc(0,I_d)\ \text{independent of }Z.
\]
Since \(X_t\mapsto \tilde X_\gamma\) is an invertible scaling, conditioning on \(X_t\) is
equivalent to conditioning on \(\tilde X_\gamma\). Moreover,
\[
\dd\gamma=\frac{\dd\tau}{(T-\tau)^2}.
\]
Therefore \eqref{eq:KL_as_mmse_tau} becomes
\begin{equation}\label{eq:KL_as_mmse_gamma}
\E_B\!\Big[\KL\!\big(\PP^{Y^{*,B}}\big\|\PP^{\hat Y^{\,B}}\big)\Big]
\le
\frac12\int_{\gamma^*}^{\infty}
\Big(
\mmse(U\mid \tilde X_\gamma)-\mmse(U\mid \tilde X_\gamma,B)
\Big)\,\dd\gamma,
\end{equation}
where \(\gamma^*:=1/(T-\tau^*)\).

Define
\[
\Phi(\gamma):=I(U;\tilde{X}_\gamma)-I(U;\tilde{X}_\gamma\mid B).
\]

Conditioning on $B$ turns $\tilde{X}_\gamma=\sqrt{\gamma}(U+B)+\Xi$ into an AWGN
channel in $U$ with a known (measurable) shift, so Lemma~\ref{lem:conditional_immse}
(with $X=U$, $S=B$ and $Y_\gamma=\tilde{X}_\gamma$) yields
\begin{equation}\label{eq:dI_cond}
\frac{\dd}{\dd\gamma}I(U;\tilde{X}_\gamma\mid B)
=\frac12\,\mmse(U\mid \tilde{X}_\gamma,B).
\end{equation}

Next, use the chain rule
\begin{equation}\label{eq:chain_Z}
I(Z;\tilde{X}_\gamma)=I(U;\tilde{X}_\gamma)+I(B;\tilde{X}_\gamma\mid U).
\end{equation}
Since $\tilde{X}_\gamma=\sqrt{\gamma}Z+\Xi$ is AWGN in $Z$, so by I-MMSE we have 
\[
\frac{\dd}{\dd\gamma}I(Z;\tilde{X}_\gamma)=\frac12\,\mmse(Z\mid \tilde{X}_\gamma).
\]
Also, given $U$, the observation $\tilde{X}_\gamma$ is an AWGN channel in $B$ with a known shift,
so Lemma~\ref{lem:conditional_immse} (with $X=B$, $S=U$) yields
\[
\frac{\dd}{\dd\gamma}I(B;\tilde{X}_\gamma\mid U)=\frac12\,\mmse(B\mid \tilde{X}_\gamma,U).
\]
Differentiating \eqref{eq:chain_Z} and subtracting the last display from the derivative of $I(Z;\tilde{X}_\gamma)$ gives
\begin{equation}\label{eq:dI_U_identity_clean_rewrite}
\frac{\dd}{\dd\gamma}I(U;\tilde{X}_\gamma)
=\frac12\Big(\mmse(Z\mid \tilde{X}_\gamma)-\mmse(B\mid \tilde{X}_\gamma,U)\Big).
\end{equation}
Because $U$ and $B$ live in orthogonal subspaces and $Z=U+B$,
\[
\mmse(Z\mid \tilde{X}_\gamma)
=\mmse(U\mid \tilde{X}_\gamma)+\mmse(B\mid \tilde{X}_\gamma),
\]
hence \eqref{eq:dI_U_identity_clean_rewrite} becomes exactly
\begin{equation}\label{eq:dI_U_identity_clean}
\frac{\dd}{\dd\gamma}I(U;\tilde{X}_\gamma)
=\frac12\Big(\mmse(U\mid \tilde{X}_\gamma)+\mmse(B\mid \tilde{X}_\gamma)-\mmse(B\mid \tilde{X}_\gamma,U)\Big).
\end{equation}

Subtracting \eqref{eq:dI_cond} from \eqref{eq:dI_U_identity_clean} yields
\begin{equation}\label{eq:dPhi}
\frac{\dd}{\dd\gamma}\Phi(\gamma)
=\frac12\Big(\mmse(U\mid \tilde{X}_\gamma)-\mmse(U\mid \tilde{X}_\gamma,B)\Big)
+\frac12\Big(\mmse(B\mid \tilde{X}_\gamma)-\mmse(B\mid \tilde{X}_\gamma,U)\Big).
\end{equation}
Insert \eqref{eq:dPhi} into \eqref{eq:KL_as_mmse_gamma} to obtain the exact decomposition
\begin{equation}\label{eq:KL_equals_boundary_minus_A_clean}
\E_B\!\Big[\KL\!\big(\PP^{Y^{*,B}}\big\|\PP^{\hat Y^{\,B}}\big)\Big]
\leq \Big[\Phi(\gamma)\Big]_{\gamma=\gamma^*}^{\infty}-A,
\end{equation}
where
\begin{equation*}
A:=\frac12\int_{\gamma^*}^{\infty}
\Big(\mmse(B\mid \tilde{X}_\gamma)-\mmse(B\mid \tilde{X}_\gamma,U)\Big)\,\dd\gamma.
\end{equation*}
By the orthogonality principle / law of total variance,
\[
\mmse(B\mid \tilde{X}_\gamma)-\mmse(B\mid \tilde{X}_\gamma,U)
=\E\!\left[\big\|\E[B\mid \tilde{X}_\gamma,U]-\E[B\mid \tilde{X}_\gamma]\big\|_2^2\right]\ge 0,
\]
so $A\ge 0$.

Using the identity $I(U;X)-I(U;X\mid B)=I(U;B)-I(U;B\mid X)$ (a direct consequence of the chain rule),
we have
\[
\Phi(\gamma)=I(U;B)-I(U;B\mid \tilde{X}_\gamma).
\]
Then,
\begin{equation}\label{eq:Phi_boundary}
\Big[\Phi(\gamma)\Big]_{\gamma=\gamma^*}^{\infty}
\leq I(U;B\mid \tilde{X}_{\gamma^*}).
\end{equation}
Combining \eqref{eq:KL_equals_boundary_minus_A_clean} and \eqref{eq:Phi_boundary} gives
\begin{equation}\label{eq:KL_upper_clean}
\E_B\!\Big[\KL\!\big(\PP^{Y^{*,B}}\big\|\PP^{\hat Y^{\,B}}\big)\Big]
\leq I(U;B\mid \tilde{X}_{\gamma^*})-A
\ \le\ I(U;B\mid \tilde{X}_{\gamma^*}).
\end{equation}

Now we want to give a lower bound for \eqref{eq:KL_as_mmse_gamma}. This is equal to the lower bound
\begin{equation*}
 \frac12\int_{\gamma^*}^{\gamma_{max}}
\Big(\mmse(U\mid \tilde{X}_\gamma)-\mmse(U\mid \tilde{X}_\gamma,B)\Big)\,\dd\gamma,
\end{equation*}
where $\gamma_{max} = \frac{1}{t_0}$. Based on what we hade then we only need to lower bound 
\begin{equation*}
    I(U;B\mid \tilde{X}_{\gamma^*}) - I(U;B\mid \tilde{X}_{\gamma_{max}}) - A_{\gamma_{max}}
\end{equation*}
Where 
\begin{equation}\label{eq:def_A_clean}
A_{\gamma_{max}}:=\frac12\int_{\gamma^*}^{\gamma_{max}}
\Big(\mmse(B\mid \tilde{X}_\gamma)-\mmse(B\mid \tilde{X}_\gamma,U)\Big)\,\dd\gamma.
\end{equation}
Decompose the observation into orthogonal components
\[
\tilde X_\gamma^\perp := P_\perp \tilde X_\gamma,\qquad
\tilde X_\gamma^\parallel := P_\parallel \tilde X_\gamma,
\qquad \tilde X_\gamma=\tilde X_\gamma^\perp+\tilde X_\gamma^\parallel.
\]
Since $\tilde X_\gamma=\sqrt{\gamma}(U+B)+\Xi$ with $\Xi\sim\Nc(0,I_d)$ independent of $(U,B)$,
we have
\[
\tilde X_\gamma^\perp=\sqrt{\gamma}\,B+\Xi^\perp,\qquad
\tilde X_\gamma^\parallel=\sqrt{\gamma}\,U+\Xi^\parallel,
\]
where $\Xi^\perp:=P_\perp\Xi$ and $\Xi^\parallel:=P_\parallel\Xi$ are independent and independent of $(U,B)$
(because they are orthogonal projections of a standard Gaussian).

\smallskip
\noindent\emph{Key claim:} conditioning on $U$, the parallel observation carries no information about $B$, hence
\begin{equation}\label{eq:mmse_condU_parallel_irrelevant}
\mmse(B\mid \tilde X_\gamma,U)=\mmse(B\mid \tilde X_\gamma^\perp,U).
\end{equation}
Indeed, given $U$, we can write $\tilde X_\gamma^\parallel=\sqrt{\gamma}\,U+\Xi^\parallel$ as a function of $U$
plus independent noise $\Xi^\parallel$, so $\tilde X_\gamma^\parallel\perp\!\!\!\perp (B,\tilde X_\gamma^\perp)\mid U$.
Therefore
\[
\Law(B\mid U,\tilde X_\gamma^\perp,\tilde X_\gamma^\parallel)=\Law(B\mid U,\tilde X_\gamma^\perp),
\]
which implies $\E[B\mid U,\tilde X_\gamma]=\E[B\mid U,\tilde X_\gamma^\perp]$ and thus \eqref{eq:mmse_condU_parallel_irrelevant}.

\smallskip
Next, by monotonicity of MMSE with respect to side information (conditioning on more cannot increase MMSE),
\begin{equation}\label{eq:mmse_more_info}
\mmse(B\mid \tilde X_\gamma)\le \mmse(B\mid \tilde X_\gamma^\perp),
\end{equation}
since $\sigma(\tilde X_\gamma^\perp)\subseteq \sigma(\tilde X_\gamma)$.

Combining \eqref{eq:mmse_condU_parallel_irrelevant} and \eqref{eq:mmse_more_info} yields the \emph{correct} pointwise bound
\begin{equation}\label{eq:mmse_gap_bound_perp_Sorrect}
\mmse(B\mid \tilde X_\gamma)-\mmse(B\mid \tilde X_\gamma,U)
\ \le\
\mmse(B\mid \tilde X_\gamma^\perp)-\mmse(B\mid \tilde X_\gamma^\perp,U).
\end{equation}
Plugging \eqref{eq:mmse_gap_bound_perp_Sorrect} into \eqref{eq:def_A_clean} gives $A_{\gamma_{max}}\le A^\perp$, where
\[
A^\perp:=\frac12\int_{\gamma^*}^{\infty}
\Big(\mmse(B\mid \tilde X_\gamma^\perp)-\mmse(B\mid \tilde X_\gamma^\perp,U)\Big)\,d\gamma.
\]

Finally, $\tilde X_\gamma^\perp=\sqrt{\gamma}\,B+\Xi^\perp$ is a Gaussian channel for $B$ (in the normal subspace);
applying Lemma~\ref{lem:conditional_immse} (after identifying an orthonormal basis of $\mathrm{range}(P_\perp)$, if desired)
gives
\[
A^\perp=\Big[I(B;\tilde X_\gamma^\perp)-I(B;\tilde X_\gamma^\perp\mid U)\Big]_{\gamma=\gamma^*}^{\infty}
\leq I(U;B\mid \tilde X_{\gamma^*}^\perp),
\]
Thus $A_{\gamma_{max}}\le I(U;B\mid \tilde X_{\gamma^*}^\perp)$, completing the lower bound.

\end{proof}

\section{Proof of Theorem~\ref{thm:total_KL_shannon_small}}

\noindent
The proof follows the same decomposition as the sampler. At the safe time \(t^*\), our
initialization is not the exact conditional law \(\Law(X_{t^*}\mid B=b)\), but a surrogate
law obtained by sampling the correct noisy normal component and then sampling the tangent
component from the unconditional slice given that normal observation. Thus the error splits
into an initialization term at time \(t^*\) and a pathwise term accumulated during the
reverse dynamics. The pathwise part is already controlled by
Theorem~\ref{thm:avg-path-kl}, so it remains to control the initialization discrepancy.
For this, we write both the true and surrogate tangent laws as mixtures over the discrete
latent normal code \(S^\perp=P_\perp S\), reduce the resulting KL divergence to a posterior
resampling error through a coupling inequality for mixture KL, and then bound that posterior
resampling error by Shannon entropy using the Gaussian I--MMSE identity.

\begin{lemma}
\label{lem:mixture_kl_coupling}
Let \(\{r_c\}_{c\in\mathcal C}\) be a family of probability measures on a measurable space
\((E,\mathcal E)\), where \(\mathcal C\) is countable. Let \(\alpha,\beta\) be probability
mass functions on \(\mathcal C\), and define
\[
\mu:=\sum_{c\in\mathcal C}\alpha(c)\,r_c,
\qquad
\nu:=\sum_{c\in\mathcal C}\beta(c)\,r_c.
\]
Then
\[
\KL(\mu\|\nu)
\le
\inf_{\lambda\in\Gamma(\alpha,\beta)}
\sum_{c,\tilde c\in\mathcal C}\lambda(c,\tilde c)\,
\KL(r_c\|r_{\tilde c}),
\]
where \(\Gamma(\alpha,\beta)\) denotes the set of couplings of \(\alpha\) and \(\beta\).
\end{lemma}

\begin{proof}
Fix any coupling \(\lambda\in\Gamma(\alpha,\beta)\). Define two probability measures on
\(\mathcal C\times\mathcal C\times E\) by
\[
\mathcal{P}_\lambda(c,\tilde c,dx):=\lambda(c,\tilde c)\,r_c(dx),
\qquad
Q_\lambda(c,\tilde c,dx):=\lambda(c,\tilde c)\,r_{\tilde c}(dx).
\]
Their \(E\)-marginals are exactly \(\mu\) and \(\nu\). Therefore, by data processing under
the projection \((c,\tilde c,x)\mapsto x\),
\[
\KL(\mu\|\nu)\le \KL(\mathcal{P}_\lambda\|Q_\lambda).
\]
Since \(\mathcal{P}_\lambda\) and \(Q_\lambda\) have the same \((c,\tilde c)\)-marginal \(\lambda\),
the chain rule for relative entropy gives
\[
\KL(\mathcal{P}_\lambda\|Q_\lambda)
=
\sum_{c,\tilde c}\lambda(c,\tilde c)\,\KL(r_c\|r_{\tilde c}).
\]
Taking the infimum over \(\lambda\in\Gamma(\alpha,\beta)\) yields the claim.
\end{proof}

\begin{lemma}
\label{lem:posterior_resample_shannon}
Let \(S\) be a discrete random variable in \(\R^m\) with \(H(C)<\infty\), and let
\[
X=S+\sigma G,
\qquad
G\sim\Nc(0,I_m),
\]
with \(G\) independent of \(S\). Let \(\tilde S\) be an independent posterior draw, i.e.
\[
\tilde S\mid X \sim \Law(S\mid X),
\qquad
\tilde S \perp S \mid X.
\]
Then
\[
\E\|S-\tilde S\|_2^2 \le 4\sigma^2 H(C).
\]
\end{lemma}

\begin{proof}
Conditional on \(X\), the random variables \(C\) and \(\tilde C\) are i.i.d. with common
law \(\Law(C\mid X)\). Hence
\[
\E\!\left[\|S-\tilde S\|_2^2\,\middle|\,X\right]
=
2\,\tr\!\big(\Cov(S\mid X)\big),
\]
so
\begin{equation}
\label{eq:resample_cov_identity}
\E\|S-\tilde S\|_2^2
=
2\,\E\tr\!\big(\Cov(S\mid X)\big).
\end{equation}
Set \(\gamma:=1/\sigma^2\) and \(Y_\gamma:=\sqrt{\gamma}\,S+G\). Since \(X=\sigma Y_\gamma\),
the observations \(X\) and \(Y_\gamma\) are equivalent, and
\[
\E\tr\!\big(\Cov(S\mid X)\big)
=
\E\tr\!\big(\Cov(S\mid Y_\gamma)\big)
=:\mmse(\gamma).
\]
For the Gaussian channel \(Y_\gamma=\sqrt{\gamma}\,S+G\), the vector I--MMSE identity gives
\[
I(S;Y_\gamma)=\frac12\int_0^\gamma \mmse(s)\,ds.
\]
Since \(\mmse(s)\) is nonincreasing in \(s\),
\[
I(S;Y_\gamma)\ge \frac{\gamma}{2}\,\mmse(\gamma).
\]
Because \(C\) is discrete,
\[
I(S;Y_\gamma)\le H(S).
\]
Thus
\[
\mmse(\gamma)\le \frac{2H(S)}{\gamma}=2\sigma^2 H(S).
\]
Substituting into \eqref{eq:resample_cov_identity} gives
\[
\E\|S-\tilde S\|_2^2 \le 4\sigma^2 H(S).
\]
\end{proof}

\begin{proof}
Let
\[
r_t^c:=\Law(X_t^\parallel\mid S^\perp=c),
\qquad c\in\Cc^\perp.
\]
Under Assumption~\ref{ass:latent_gaussian_mixture},
\[
Z=S+\varepsilon N,
\qquad
B=P_\perp Z=S^\perp +\varepsilon N^\perp,
\qquad
X_t^\perp=B+W_t^\perp.
\]
Since the tangent and normal noises are independent, conditional on \(S^\perp\) the variable
\(X_t^\parallel\) is independent of both \(B\) and \(X_t^\perp\). Therefore, if
\[
\pi_b(c):=\PP(S^\perp=c\mid B=b),
\qquad
\pi_x(c):=\PP(S^\perp=c\mid X_t^\perp=x),
\]
then
\[
\Law(X_t^\parallel\mid B=b)=\sum_{c\in\Cc^\perp}\pi_b(c)\,r_t^c,
\qquad
\Law(X_t^\parallel\mid X_t^\perp=x)=\sum_{c\in\Cc^\perp}\pi_x(c)\,r_t^c.
\]

Moreover, conditional on \(B=b\), the normal component is \(X_t^\perp=b+W_t^\perp\), and it
is independent of \(X_t^\parallel\). Hence the true conditional law and the surrogate
initialization law factorize as
\[
p_t^{*,b}(x^\perp,x^\parallel)
=
p_t(x^\perp\mid B=b)\,\mu_t^b(x^\parallel),
\qquad
\mu_t^b:=\sum_{c\in\Cc^\perp}\pi_b(c)\,r_t^c,
\]
and
\[
\hat p_t^{\,b}(x^\perp,x^\parallel)
=
p_t(x^\perp\mid B=b)\,\nu_t^{x^\perp}(x^\parallel),
\qquad
\nu_t^x:=\sum_{c\in\Cc^\perp}\pi_x(c)\,r_t^c.
\]
Therefore
\begin{equation}
\label{eq:init_kl_conditional_expectation_clean}
\KL(p_t^{*,b}\|\hat p_t^{\,b})
=
\E\!\left[\KL(\mu_t^b\|\nu_t^{X_t^\perp})\,\middle|\,B=b\right].
\end{equation}

Applying Lemma~\ref{lem:mixture_kl_coupling} and choosing the product coupling
\(\lambda=\pi_b\otimes\pi_x\), we obtain
\[
\KL(\mu_t^b\|\nu_t^x)
\le
\sum_{c,\tilde c}\pi_b(c)\pi_x(\tilde c)\,\KL(r_t^c\|r_t^{\tilde c}).
\]
By Assumption~\ref{ass:quad_log_t0_kernel},
\[
\KL(r_t^c\|r_t^{\tilde c})\le L_t\|c-\tilde c\|_2^2,
\qquad t\ge t_0,
\]
and hence
\[
\KL(\mu_t^b\|\nu_t^x)
\le
L_t\sum_{c,\tilde c}\pi_b(c)\pi_x(\tilde c)\|c-\tilde c\|_2^2.
\]
Substituting into \eqref{eq:init_kl_conditional_expectation_clean} and averaging over \(B\)
gives
\begin{equation}
\label{eq:init_kl_by_resampling_clean}
\E_B\!\left[\KL(p_t^{*,B}\|\hat p_t^{\,B})\right]
\le
L_t\,\E\|S^\perp-\tilde S^\perp\|_2^2,
\end{equation}
where, conditional on \(X_t^\perp\), the random variable \(\tilde S^\perp\) is an
independent draw from \(\Law(S^\perp\mid X_t^\perp)\).

Since
\[
X_t^\perp
=
S^\perp+\varepsilon N^\perp+W_t^\perp
=
S^\perp+\sqrt{t+\varepsilon^2}\,G,
\qquad G\sim\Nc(0,I_m),
\]
Lemma~\ref{lem:posterior_resample_shannon} yields
\[
\E\|S^\perp-\tilde S^\perp\|_2^2
\le
4(t+\varepsilon^2)H(S^\perp).
\]
Combining this with \eqref{eq:init_kl_by_resampling_clean} and setting \(t=t^*\), we obtain
\begin{equation}
\label{eq:init_shannon_bound_clean}
\E_B\!\left[\KL(p_{t^*}^{*,B}\|\hat p_{t^*}^{\,B})\right]
\le
4L_{t^*}(t^*+\varepsilon^2)H(S^\perp).
\end{equation}

For each \(b\), let \(\PP^{*,b}\) be the path measure of the true conditional reverse SDE on
\([\tau^*,T-t_0]\), started from \(p_{t^*}^{*,b}\), and let \(\hat{\PP}^{\,b}\) be the path
measure of the surrogate reverse SDE on the same interval, started from \(\hat p_{t^*}^{\,b}\).
Let \(\tilde{\PP}^{\,b}\) denote the path measure obtained by running the surrogate reverse
SDE from the true initial law \(p_{t^*}^{*,b}\). Then
\[
\KL(\PP^{*,b}\|\hat{\PP}^{\,b})
=
\KL(\PP^{*,b}\|\tilde{\PP}^{\,b})
+
\KL(p_{t^*}^{*,b}\|\hat p_{t^*}^{\,b}).
\]
Averaging over \(B\) yields
\[
\E_B\!\big[\KL(\PP^{*,B}\|\hat{\PP}^{\,B})\big]
=
\E_B\!\big[\KL(\PP^{*,B}\|\tilde{\PP}^{\,B})\big]
+
\E_B\!\big[\KL(p_{t^*}^{*,B}\|\hat p_{t^*}^{\,B})\big].
\]
By Theorem~\ref{thm:avg-path-kl},
\[
\E_B\!\big[\KL(\PP^{*,B}\|\tilde{\PP}^{\,B})\big]
\le
I\!\big(Z^\parallel;Z^\perp\mid X_{t^*}\big),
\]
and by Proposition~\ref{prop:latent_cmi_dpi},
\[
I\!\big(Z^\parallel;Z^\perp\mid X_{t^*}\big)
\le
I\!\big(S^\parallel;S^\perp\mid X_{t^*}\big).
\]
Together with \eqref{eq:init_shannon_bound_clean}, this gives
\[
\E_B\!\big[\KL(\PP^{*,B}\|\hat{\PP}^{\,B})\big]
\le
4L_{t^*}(t^*+\varepsilon^2)H(S^\perp)
+
I\!\big(S^\parallel;S^\perp\mid X_{t^*}\big).
\]

Finally, the terminal tangent marginal is a measurable image of path space, so by data
processing,
\[
\E_B\!\left[\KL\!\big(\mu_{T-t_0}^{*,B}\,\Vert\,\hat\mu_{T-t_0}^{\,B}\big)\right]
\le
\E_B\!\big[\KL(\PP^{*,B}\|\hat{\PP}^{\,B})\big].
\]
This proves \eqref{eq:total_KL_shannon_small}.
\end{proof}

\section{Proof of Theorem~\ref{thm:total_error_sep_renyi_single}}

\noindent
The proof again separates initialization and pathwise contributions, but now the
\(\delta\)-separation assumption upgrades both bounds from Shannon-scale control to
exponential control. The initialization term is handled through the same mixture
representation as above, followed by a posterior-resampling tail estimate for the effective
normal Gaussian channel. The pathwise term is bounded through
Theorem~\ref{thm:avg-path-kl}, the latent comparison proposition, and an exponential bound
on the residual conditional entropy \(H(S^\perp\mid X_{t^*}^\perp)\).

\begin{lemma}
\label{lem:tail_resampling_C}
Let \(C\) be a discrete random variable supported on a countable set
\(\Cc^\perp\subset\R^m\) with pmf \(p_S\), and define
\[
H_{1/2}(S):=2\log\sum_{c\in\Cc^\perp}\sqrt{p_S(c)}.
\]
Fix \(t>0\), and consider the Gaussian channel
\[
X\mid (S=c)\sim \Nc(c,(t+\varepsilon^2)I_m).
\]
Let \(p_t(x\mid c)\) denote this Gaussian density and let \(p_t(c\mid x)\) be the posterior
\[
p_t(c\mid x)=
\frac{p_S(c)\,p_t(x\mid c)}
{\sum_{u\in\Cc^\perp}p_S(u)\,p_t(x\mid u)}.
\]
For each \(c^*\in\Cc^\perp\), define the posterior-resampling kernel
\[
K_t(c,c^*):=\int p_t(c\mid x)\,p_t(x\mid c^*)\,dx.
\]
Let \(S^*\sim p_S\), and conditional on \(S^*=c^*\), let \(\tilde S\) have pmf
\(K_t(\cdot,c^*)\). Set
\[
R:=\|\tilde S-S^*\|_2.
\]
Then for every \(r\ge 0\),
\[
\PP(R\ge r)
\le
\frac12\exp\!\Big(H_{1/2}(S)-\frac{r^2}{8(t+\varepsilon^2)}\Big).
\]
\end{lemma}

\begin{proof}
Fix \(c,c^*\in\Cc^\perp\) and \(x\in\R^m\). By Bayes' rule,
\[
p_t(c\mid x)
\le
\frac{p_S(c)p_t(x\mid c)}
{p_S(c)p_t(x\mid c)+p_S(c^*)p_t(x\mid c^*)}.
\]
Writing \(A:=p_S(c)p_t(x\mid c)\) and \(B:=p_S(c^*)p_t(x\mid c^*)\), the inequality
\(A+B\ge 2\sqrt{AB}\) gives
\[
\frac{A}{A+B}
\le
\frac12\sqrt{\frac{A}{B}}
=
\frac12\sqrt{\frac{p_S(c)}{p_S(c^*)}}
\sqrt{\frac{p_t(x\mid c)}{p_t(x\mid c^*)}}.
\]
Multiplying by \(p_t(x\mid c^*)\) and integrating over \(x\), we obtain
\[
K_t(c,c^*)
\le
\frac12\sqrt{\frac{p_S(c)}{p_S(c^*)}}
\int \sqrt{p_t(x\mid c)p_t(x\mid c^*)}\,dx.
\]
For isotropic Gaussians with covariance \((t+\varepsilon^2)I_m\), the Hellinger affinity is
\[
\int \sqrt{p_t(x\mid c)p_t(x\mid c^*)}\,dx
=
\exp\!\Big(-\frac{\|c-c^*\|_2^2}{8(t+\varepsilon^2)}\Big).
\]
Hence
\[
K_t(c,c^*)
\le
\frac12\sqrt{\frac{p_S(c)}{p_S(c^*)}}
\exp\!\Big(-\frac{\|c-c^*\|_2^2}{8(t+\varepsilon^2)}\Big).
\]
Let
\[
M:=\sum_{c\in\Cc^\perp}\sqrt{p_S(c)},
\qquad
M^2=e^{H_{1/2}(C)}.
\]
Then
\[
\PP(R\ge r\mid S^*=c^*)
=
\sum_{\|c-c^*\|\ge r}K_t(c,c^*)
\le
\frac{M}{2\sqrt{p_S(c^*)}}
e^{-r^2/(8(t+\varepsilon^2))}.
\]
Averaging over \(S^*\sim p_S\) yields
\[
\PP(R\ge r)
\le
\frac12\exp\!\Big(H_{1/2}(S)-\frac{r^2}{8(t+\varepsilon^2)}\Big).
\]
\end{proof}

\begin{lemma}
\label{lem:cond_entropy_sep_S}
Assume Assumption~\ref{ass:delta_sep_S} and \(H_{1/2}(S^\perp)<\infty\). Then
\[
H(S^\perp\mid X_{t^*}^\perp)
\le
2\exp\!\Big(H_{1/2}(S^\perp)-\frac{\delta^2}{8(t^*+\varepsilon^2)}\Big).
\]
\end{lemma}

\begin{proof}
Write \(\pi_c:=\PP(S^\perp=c)\) for \(c\in\Cc^\perp\), and fix \(c^*\in\Cc^\perp\). Conditional on
\(S^\perp=c^*\),
\[
X_{t^*}^\perp=c^*+\sqrt{t^*+\varepsilon^2}\,G,
\qquad
G\sim\Nc(0,I_m).
\]
For \(x\in\R^m\), let
\[
p_x(c):=\PP(S^\perp=c\mid X_{t^*}^\perp=x).
\]
Then
\[
H(S^\perp\mid X_{t^*}^\perp=x)
=
\sum_{c\in\Cc^\perp} p_x(c)\log\frac{1}{p_x(c)}.
\]
For \(c\neq c^*\), define
\[
l_c(x):=\frac{p_x(c)}{p_x(c^*)},
\qquad
R(x):=\sum_{c\neq c^*} l_c(x).
\]
Then
\[
p_x(c^*)=\frac{1}{1+R(x)},
\qquad
p_x(c)=\frac{l_c(x)}{1+R(x)}\quad(c\neq c^*),
\]
and therefore
\[
H(S^\perp\mid X_{t^*}^\perp=x)
=
\sum_{c\neq c^*}\frac{l_c(x)}{1+R(x)}\log\frac{1}{l_c(x)}
+\log(1+R(x)).
\]
Using
\[
u\log\frac{1}{u}\le \sqrt{u}\quad(0<u\le 1),
\qquad
\log(1+v)\le \sqrt{v}\quad(v\ge 0),
\]
and discarding the nonpositive terms with \(r_c(x)>1\), we obtain
\[
H(S^\perp\mid X_{t^*}^\perp=x)
\le
2\sum_{c\neq c^*}\sqrt{l_c(x)}.
\]

By Bayes' rule,
\[
l_c(x)
=
\frac{\pi_c}{\pi_{c^*}}
\frac{\varphi_{\sigma_*^2}(x-c)}{\varphi_{\sigma_*^2}(x-c^*)},
\qquad
\sigma_*^2:=t^*+\varepsilon^2,
\]
where \(\varphi_{\sigma_*^2}\) is the Gaussian density with covariance \(\sigma_*^2I_m\).
Writing \(x=c^*+w\), we obtain
\[
\sqrt{l_c(x)}
=
\sqrt{\frac{\pi_c}{\pi_{c^*}}}
\exp\!\left(
-\frac{\|c-c^*\|_2^2}{4\sigma_*^2}
+
\frac{\langle w,c-c^*\rangle}{2\sigma_*^2}
\right).
\]
Taking expectation over \(w\sim\Nc(0,\sigma_*^2I_m)\),
\[
\E\!\left[\sqrt{l_c(X_{t^*}^\perp)}\,\middle|\,S^\perp=c^*\right]
=
\sqrt{\frac{\pi_c}{\pi_{c^*}}}
\exp\!\left(-\frac{\|c-c^*\|_2^2}{8\sigma_*^2}\right).
\]
By \(\delta\)-separation,
\[
\|c-c^*\|_2\ge \delta
\qquad(c\neq c^*),
\]
and hence
\[
\E\!\left[\sqrt{l_c(X_{t^*}^\perp)}\,\middle|\,S^\perp=c^*\right]
\le
\sqrt{\frac{\pi_c}{\pi_{c^*}}}
e^{-\delta^2/(8\sigma_*^2)}.
\]
Therefore
\[
\E\!\left[H(S^\perp\mid X_{t^*}^\perp)\,\middle|\,S^\perp=c^*\right]
\le
2e^{-\delta^2/(8\sigma_*^2)}
\sum_{c\neq c^*}\sqrt{\frac{\pi_c}{\pi_{c^*}}}.
\]
Averaging over \(S^\perp\) gives
\[
H(S^\perp\mid X_{t^*}^\perp)
\le
2e^{-\delta^2/(8\sigma_*^2)}
\sum_{c^*}\sqrt{\pi_{c^*}}
\sum_{c\neq c^*}\sqrt{\pi_c}
\le
2e^{-\delta^2/(8\sigma_*^2)}
\Big(\sum_{c\in\Cc^\perp}\sqrt{\pi_c}\Big)^2.
\]
Since
\[
\Big(\sum_{c\in\Cc^\perp}\sqrt{\pi_c}\Big)^2=e^{H_{1/2}(S^\perp)},
\]
the claim follows.
\end{proof}

\begin{proof}
Let
\[
\sigma_*^2:=t^*+\varepsilon^2,
\qquad
H_{1/2}:=H_{1/2}(S^\perp).
\]
As in the proof of Theorem~\ref{thm:total_KL_shannon_small}, the true and surrogate tangent
laws at time \(t\) can be written as mixtures over \(S^\perp\), and Lemma~\ref{lem:mixture_kl_coupling}
therefore yields
\[
\KL(p_t^{*,b}\|\hat p_t^{\,b})
\le
\sum_{c,\tilde c}\pi_b(c)\pi_x(\tilde c)\,\KL(r_t^c\|r_t^{\tilde c}).
\]
Using Assumption~\ref{ass:quad_log_t0_kernel} at time \(t=t^*\), we get
\[
\KL(r_{t^*}^c\|r_{t^*}^{\tilde c})
\le
L_{t^*}\|c-\tilde c\|_2^2.
\]
Hence
\[
\E_B\!\Big[\KL\!\big(p_{t^*}^{*,B}\,\Vert\,\hat p_{t^*}^{\,B}\big)\Big]
\le
L_{t^*}\,\E\|{S^*}^{\perp}-\tilde S^\perp\|_2^2,
\]
where \({S^*}^{\perp}\sim p_{S^\perp}\), and conditional on \({S^*}^{\perp}=c^*\), the variable
\(\tilde S^\perp\) is drawn from the posterior-resampling kernel of the effective channel
\[
X_{t^*}^\perp\mid(S^\perp=c)\sim\Nc(c,\sigma_*^2I_m).
\]

Let
\[
R:=\|{S^*}^{\perp}-\tilde S^\perp\|_2.
\]
By Assumption~\ref{ass:delta_sep_S}, either \(R=0\) or \(R\ge \delta\). Therefore
\[
\E[R^2]
=
\int_0^\infty \PP(R^2\ge s)\,ds
=
\int_0^{\delta^2}\PP(R\ge\delta)\,ds
+
\int_{\delta^2}^\infty \PP(R\ge \sqrt{s})\,ds.
\]
Applying Lemma~\ref{lem:tail_resampling_C},
\[
\PP(R\ge r)
\le
\frac12\exp\!\Big(H_{1/2}-\frac{r^2}{8\sigma_*^2}\Big),
\]
we obtain
\[
\E[R^2]
\le
\frac{\delta^2}{2}\exp\!\Big(H_{1/2}-\frac{\delta^2}{8\sigma_*^2}\Big)
+
\int_{\delta^2}^{\infty}
\frac12\exp\!\Big(H_{1/2}-\frac{s}{8\sigma_*^2}\Big)\,ds.
\]
Evaluating the integral yields
\[
\E[R^2]
\le
\Big(\frac{\delta^2}{2}+4\sigma_*^2\Big)
\exp\!\Big(H_{1/2}-\frac{\delta^2}{8\sigma_*^2}\Big),
\]
and therefore
\[
\E_B\!\Big[\KL\!\big(p_{t^*}^{*,B}\,\Vert\,\hat p_{t^*}^{\,B}\big)\Big]
\le
L_{t^*}\Big(\frac{\delta^2}{2}+4\sigma_*^2\Big)
\exp\!\Big(H_{1/2}-\frac{\delta^2}{8\sigma_*^2}\Big).
\]

For the pathwise term, Theorem~\ref{thm:avg-path-kl} gives
\[
\E_B\!\Big[\KL\!\big(\PP^{Y^{*,B}}\Vert \PP^{\tilde Y^{\,B}}\big)\Big]
\le
I\!\big(Z^\parallel;Z^\perp\mid X_{t^*}\big).
\]
Using Proposition~\ref{prop:latent_cmi_dpi},
\[
I\!\big(Z^\parallel;Z^\perp\mid X_{t^*}\big)
\le
I\!\big(S^\parallel;S^\perp\mid X_{t^*}\big).
\]
Since \(S^\perp=S^\perp\),
\[
I\!\big(S^\parallel;S^\perp\mid X_{t^*}\big)
=
I\!\big(S^\parallel;S^\perp\mid X_{t^*}\big)
\le
H(S^\perp\mid X_{t^*})
\le
H(S^\perp\mid X_{t^*}^\perp).
\]
Lemma~\ref{lem:cond_entropy_sep_S} now implies
\[
H(S^\perp\mid X_{t^*}^\perp)
\le
2\exp\!\Big(H_{1/2}-\frac{\delta^2}{8\sigma_*^2}\Big).
\]
Hence
\[
\E_B\!\Big[\KL\!\big(\PP^{Y^{*,B}}\Vert \PP^{\tilde Y^{\,B}}\big)\Big]
\le
2\exp\!\Big(H_{1/2}-\frac{\delta^2}{8\sigma_*^2}\Big).
\]

Combining the initialization and pathwise bounds yields
\[
\E_B\!\Big[
\KL\!\big(p_{t^*}^{*,B}\,\Vert\,\hat p_{t^*}^{\,B}\big)
+
\KL\!\big(\PP^{Y^{*,B}}\Vert \PP^{\hat Y^{\,B}}\big)
\Big]
\le
L_{t^*}\Big(\frac{\delta^2}{2}+4\sigma_*^2\Big)
\exp\!\Big(H_{1/2}-\frac{\delta^2}{8\sigma_*^2}\Big)
+
2\exp\!\Big(H_{1/2}-\frac{\delta^2}{8\sigma_*^2}\Big),
\]
which is \eqref{eq:total_error_sep_renyi_single}.

Finally, the terminal tangent marginal is a measurable image of path space, so by data
processing,
\[
\E_B\!\Big[
\KL\!\big(\mu_{T-t_0}^{*,B}\,\Vert\,\hat\mu_{T-t_0}^{\,B}\big)
\Big]
\le
\E_B\!\Big[
\KL\!\big(p_{t^*}^{*,B}\,\Vert\,\hat p_{t^*}^{\,B}\big)
+
\KL\!\big(\PP^{Y^{*,B}}\Vert \PP^{\hat Y^{\,B}}\big)
\Big].
\]
This proves \eqref{eq:terminal_sep_renyi_single}.
\end{proof}

\section{DDPM implementation of the VP normal correction}
\label{app:vp-ddpm-normal-correction}

This appendix records the VP/DDPM form of the normal correction used in the
experiments. Consider the forward marginal
\[
X_t=\alpha_t Z+\sigma_t\xi,
\qquad
\xi\sim\mathcal N(0,I_d),
\]
and condition on \(B=P_\perp Z=b\). Let \(p_t^{*,b}\) be the density of
\(X_t\mid B=b\), with score \(s_t^{*,b}=\nabla\log p_t^{*,b}\). The VP Tweedie
identity gives
\[
s_t^{*,b}(x)
=
\frac{\alpha_t\E[Z\mid X_t=x,B=b]-x}{\sigma_t^2}.
\]
Projecting onto the normal space and using \(P_\perp Z=b\) under the
conditioning,
\[
P_\perp\E[Z\mid X_t=x,B=b]=b,
\]
we obtain
\[
P_\perp s_t^{*,b}(x)
=
\frac{\alpha_t b-P_\perp x}{\sigma_t^2}.
\]
Thus the normal correction used in the DDPM implementation is
\[
\frac{\alpha_t b-P_\perp x_t}{\sigma_t^2}.
\]

We next relate this expression to the DDNM-style projected denoising update.
For a pretrained VP/DDPM model, the usual Tweedie denoiser is
\[
\hat z_0(x_t)
=
\frac{x_t+\sigma_t^2s_t(x_t)}{\alpha_t}.
\]
DDNM replaces the normal component of this denoised estimate by the observed
level \(b\):
\[
\tilde z_0(x_t;y)
=
P_\parallel \hat z_0(x_t)+b.
\]
The score associated with this projected denoiser is obtained by inverting
Tweedie's formula:
\[
\hat s_t^{\rm DDNM}(x_t;y)
=
\frac{\alpha_t\tilde z_0(x_t;y)-x_t}{\sigma_t^2}.
\]
Substituting the expression for \(\tilde z_0\),
\[
\begin{aligned}
\hat s_t^{\rm DDNM}(x_t;y)
&=
\frac{
\alpha_tP_\parallel\hat z_0(x_t)+\alpha_t b-x_t
}{\sigma_t^2} \\
&=
\frac{
P_\parallel x_t+\sigma_t^2P_\parallel s_t(x_t)
+\alpha_t b
-P_\parallel x_t-P_\perp x_t
}{\sigma_t^2} \\
&=
P_\parallel s_t(x_t)
+
\frac{\alpha_t b-P_\perp x_t}{\sigma_t^2}.
\end{aligned}
\]
Therefore the DDPM/DDIM implementation used in the experiments applies the
closed-form VP normal correction together with the pretrained tangent score.

\bibliography{new_ref}
\end{document}